\documentclass{article}

\usepackage{microtype}
\usepackage{graphicx}
\usepackage{subfigure}
\usepackage{booktabs} %

\usepackage{hyperref}

\usepackage[accepted]{icml2024}

\usepackage{amsmath}
\usepackage{amssymb}
\usepackage{mathtools}
\usepackage{amsthm}

\usepackage[capitalize,noabbrev]{cleveref}

\theoremstyle{plain}
\newtheorem{theorem}{Theorem}[section]

\newtheorem{lemma}[theorem]{Lemma}

\theoremstyle{definition}
\newtheorem{definition}[theorem]{Definition}
\newtheorem{assumption}[theorem]{Assumption}
\theoremstyle{remark}
\newtheorem{remark}[theorem]{Remark}

\usepackage{multirow}
\usepackage{placeins} %
\usepackage{bm}
\def\1{\bm{1}}
\newcommand{\modification}[1]{\textcolor{blue}{#1}}
\renewcommand{\modification}[1]{#1}

\icmltitlerunning{StableSSM: Alleviating the Curse of Memory in State-space Models through Stable Reparameterization}

\begin{document}

\twocolumn[
    \icmltitle{StableSSM: Alleviating the Curse of Memory in State-space Models through Stable Reparameterization}

\icmlsetsymbol{equal}{*}

\begin{icmlauthorlist}
    \icmlauthor{Shida Wang}{nus}
    \icmlauthor{Qianxiao Li}{nus,ifim}
\end{icmlauthorlist}
    
\icmlaffiliation{nus}{Department of Mathematics, National University of Singapore}
\icmlaffiliation{ifim}{Institute for Functional Intelligent Materials, National University of Singapore}

\icmlcorrespondingauthor{Qianxiao Li}{qianxiao@nus.edu.sg}

\icmlkeywords{Machine Learning, ICML, State-space Models, Curse of Memory}

\vskip 0.3in
]

\printAffiliationsAndNotice{}  %

\begin{abstract}
    In this paper, we investigate the long-term memory learning capabilities of state-space models (SSMs) from the perspective of parameterization. 
    We prove that state-space models without any reparameterization exhibit a memory limitation similar to that of traditional RNNs: the target relationships that can be stably approximated by state-space models must have an exponential decaying memory. 
    Our analysis identifies this ``curse of memory'' as a result of the recurrent weights converging to a stability boundary, suggesting that a reparameterization technique can be effective. 
    To this end, we introduce a class of reparameterization techniques for SSMs that effectively lift its memory limitations. 
    Besides improving approximation capabilities, we further illustrate that a principled choice of reparameterization scheme can also enhance optimization stability. 
    We validate our findings using synthetic datasets, language models and image classifications. 
\end{abstract}

\section{Introduction}
\label{introduction}

Understanding long-term memory relationships is fundamental in sequence modeling. 
Capturing this prolonged memory is vital, especially in applications like time series prediction~\citep{connor1994.RecurrentNeuralNetworksa}, language models~\citep{sutskever.GeneratingTextRecurrent}.
Since its emergence, transformers~\citep{vaswani2017.AttentionAllYou} have become the go-to models for language representation tasks~\citep{brown2020.LanguageModelsArea}.
However, a significant drawback lies in their computational complexity, which is asymptotically $O(T^2)$, where $T$ is the sequence length.
This computational bottleneck has been a critical impediment to the further scaling-up of transformer models. 
State-space models such as S4~\citep{gu2022.EfficientlyModelingLonga}, S5~\citep{smith2023.SimplifiedStateSpace}, LRU~\citep{orvieto2023.ResurrectingRecurrentNeurald}, RWKV~\citep{peng2023rwkv}, RetNet~\citep{sun2023retentive} and Mamba~\citep{gu2023.MambaLinearTimeSequence} offer an alternative approach. 
These models are of the recurrent type and excel in long-term memory learning. 
Their architecture is specifically designed to capture temporal dependencies over extended sequences, providing a robust solution for tasks requiring long-term memory~\citep{tay2021.LongRangeArena}.
One of the advantages of state-space models over traditional RNNs lies in their computational efficiency, achieved through the application of parallel scan algorithms~\citep{martin2018.ParallelizingLinearRecurrent} and Fast Fourier Transform (FFT)~\citep{tolimieri1989.AlgorithmsDiscreteFourier,gu2022.EfficientlyModelingLonga}. 
Traditional nonlinear RNNs are often plagued by slow forward and backward propagation, a limitation that state-space models circumvent by leveraging linear RNN blocks.

Traditional linear/nonlinear RNNs exhibit an asymptotically exponential decay in memory~\citep{wang2023.InverseApproximationTheory}. 
This phenomenon explains the difficulty in both approximation and optimization to learn long-term memory using RNNs (also named curse of memory). 
In practice, empirical results show that SSMs variants like S4 overcome some of the memory issues.
\textbf{The previous empirical results suggest that either (i) the ``linear dynamics and nonlinear layerwise activation'' or (ii) the parameterization inherent to S4, is pivotal in achieving the enhanced performance.}
Current research answers which one is more important. 
We first prove an inverse approximation theorem showing that state-space models without reparameterization still suffer from the ``curse of memory'', which is consistent with empirical results~\citep{wang2023.StatespaceModelsLayerwisea}. 
This rules out the point (i) as the reason for SSMs' good long-term memory learning. 
A natural question arises regarding whether the reparameterizations are the key to learn long-term memory. 
We prove a class of reparameterization functions $f$, which we call stable reparameterization, enables the stable approximation of nonlinear functionals. 
This includes commonly used exponential reparameterization and softplus reparameterization. 
Furthermore, we question whether S4's parameterizations are optimal. 
Here we give a particular sense in terms of optimization stability that they are not optimal. 
We propose the optimal one and show its stability via numerical experiments.

We summarize our main contributions as follow:
\begin{enumerate}
    \item We prove that similar to RNNs, the state-space models without reparameterization can only stably approximate targets with exponential decaying memory. 
    \item We identify a class of stable reparameterization which achieves the stable approximation of \textbf{any nonlinear functionals}. 
    Both theoretical and empirical evidence highlight that stable reparameterization is crucial for long-term memory learning.
    \item From the optimization viewpoint, we propose the gradient boundedness as the criterion and show the gradients are bounded by a form that depends on the parameterization. 
    Based on the gradient bound, we solve the differential equation and derive the \textbf{``best''  reparameterization in the stability sense} and verify the stability of this new reparameterization across different parameterization schemes. 
\end{enumerate}

\paragraph{Notation.}
We use the bold face to represent the sequence while then normal letters are scalars, vectors or functions. 
Throughout this paper we use $\|\cdot\|$ to denote norms over sequences of vectors, or function(al)s, while $|\cdot|$ (with subscripts) represents the norm of number, vector or weights tuple.
Here $|x|_\infty := \max_{i} |x_i|, |x|_2 := \sqrt{\sum_{i} x_i^2}, |x|_1 := \sum_{i} |x_i|$ are the usual max ($L_{\infty}$) norm, $L_2$ norm and $L_1$ norm. We use $m$ to denote the hidden dimension.

\section{Background}
\label{background}

In this section, we first introduce the state-space models and compare them to traditional nonlinear RNNs.
Subsequently, we adopt the sequence modeling as a problem in nonlinear functional approximation framework. 
Specifically, the theoretical properties we anticipate from the targets are defined.
Moreover, we define the ``curse of memory'' phenomenon and provide a concise summary of prior theoretical definitions and results concerning RNNs.

\subsection{State-space models}
State-space models (SSMs) are a family of neural networks specialized in sequence modeling.
Unlike Recurrent Neural Networks (RNNs)~\citep{rumelhart1986.LearningRepresentationsBackpropagating}, SSMs have layer-wise nonlinearity and linear dynamics within their hidden states.
This unique structure facilitates accelerated computing using FFT~\citep{gu2022.EfficientlyModelingLonga} or parallel scan~\citep{martin2018.ParallelizingLinearRecurrent}.
With trainable weights $W \in \mathbb{R}^{m \times m}, U \in \mathbb{R}^{m \times d}, b,c \in \mathbb{R}^m$ and activation function $\sigma(\cdot)$, the simplest SSM maps $d$-dimensional input sequence $\mathbf{x}=\{x_t\}$ to 1-dimensional output sequence $\{\hat{y}_t\}$.
To simplify our analysis, we utilize the continuous-time framework referenced in \citet{li2020.CurseMemoryRecurrent}:
\begin{align}
\label{eq:ssm}
\begin{array}{rll}
    \frac{dh_t}{dt}   & = Wh_t + Ux_t + b,             & h_{-\infty} = 0, \\
    \hat{y}_t         & = c^\top \bm{\sigma}(h_t),     & t \in \mathbb{R}. 
\end{array}
\end{align}
As detailed in \cref{sec:graphical_demonstration_eq2}, the above form is a simplification of practical SSMs in the sense that practical SSMs can be realized by the stacking of \cref{eq:ssm}.

It is known that multi-layer state-space models are universal approximators~\citep{wang2023.StatespaceModelsLayerwisea,orvieto2023universality}. 
In particular, when the nonlinearity is added layer-wise, it is sufficient (in approximation sense) to use \emph{real diagonal $W$}~\citep{gu2022.ParameterizationInitializationDiagonal,li2022.ApproximationOptimizationTheory}.
In this paper, we only consider the real diagonal matrix case and denote it by $\Lambda = \textrm{Diag}(\lambda_1, \dots, \lambda_m)$. 
\begin{align}
\label{eq:2}
    \frac{dh_t}{dt} & = \Lambda h_t + Ux_t + b.
\end{align}
Compared with S4, the major differences lie in initialization such as HiPPO~\citep{gu2020.HiPPORecurrentMemorya} and parameters saving method such as DPLR~\citep{gu2022.ParameterizationInitializationDiagonal} and NPLR~\citep{gu2022.EfficientlyModelingLonga}.

\subsection{Sequence modeling as nonlinear functional approximations}

Sequence modeling aims to discern the association between an input series, represented as $\mathbf{x}=\{x_t\}$, and its corresponding output series, denoted as $\mathbf{y}=\{y_t\}$.
The input series are continuous bounded inputs vanishing at infinity: $\mathbf{x} \in \mathcal{X} = C_0(\mathbb{R}, \mathbb{R}^d)$ with norm $\|\mathbf{x}\|_\infty := \sup_{t\in\mathbb R} |x_t|_\infty$.
It is assumed that the input and output sequences are determined from the inputs via a set of functionals, symbolized as
\begin{equation}
    \mathbf{H} = \{H_t: \mathcal{X} \rightarrow \mathbb{R}: t \in \mathbb{R} \},
\end{equation}
through the relationship $y_t = H_t(\mathbf{x})$.
In essence, the challenge of sequential approximation boils down to estimating the desired functional sequence $\mathbf{H}$ using a different functional sequence $\mathbf{\widehat H}$ potentially from a predefined model space such as SSMs.

In this paper we focus on target functionals that are bounded, causal, continuous, regular, time-homogeneous (time-shift invariant). 
Formal definitions are given in \cref{subsec:properties_of_targets}. 
The continuity, boundedness, time-homogeneity, causality are important properties for good sequence-to-sequence models to have. 
Linearity is an important simplification as many theoretical theorems are available in functional analysis~\citep{stein2003.PrincetonLecturesAnalysis}. 
Without loss of generality, we assume that the nonlinear functionals satisfy $H_t(\mathbf{0}) = 0$. 
It can be achieved via studying $H_t^{\textrm{adjusted}}(\mathbf{x}) = H_t(\mathbf{x}) - H_t(\mathbf{0})$.

\subsection{Memory function, stable approximation and curse of memory}

The concept of memory has been extensively explored in academic literature, yet much of previous works rely on heuristic approaches and empirical testing, particularly in the context of learning long-term memory~\citep{poli2023.HyenaHierarchyLarger}.
Here we study the memory property from a theoretical perspective. 

Our study employs the extended framework proposed by \citet{wang2023.InverseApproximationTheory}, which specifically focuses on nonlinear RNNs.
However, these studies do not address the case of state-space models.
Within the same framework, the slightly different memory function and decaying memory concepts enable us to explore the approximation capabilities of nonlinear functionals using SSMs.

\begin{definition}[Memory function]
    For bounded, causal, continuous, regular and time-homogeneous nonlinear functional sequences $\mathbf{H} = \{H_t: t \in \mathbb{R}\}$ on $\mathcal{X}$, define the following function as the \emph{memory function} of $\mathbf{H}$: Over bounded Heaviside input $\mathbf{u}^{x}(t) = x \cdot \1_{\{t \geq 0\}}$
    \begin{equation}
        \mathcal{M}(\mathbf{H})(t) := \sup_{x \neq 0} \frac{\left | \frac{d}{dt}H_t(\mathbf{u}^{x}) \right |}{|x|_{\infty} + 1}.
    \end{equation}
    We add 1 in the memory function definition to make it more regular. 
    The memory function of the target functionals is assumed to be finite for all $t \in \mathbb{R}$. 
\end{definition}

\begin{definition}[Decaying memory]
    The functional sequences $\mathbf{H}$ has a \emph{decaying memory} if 
    \begin{equation}
        \lim_{t \to \infty} \mathcal{M}(\mathbf{H})(t) = 0.
    \end{equation}
    In particular, we say it has an \emph{exponential (polynomial) decaying memory} if there exists constant $\beta > 0$ such that $\lim_{t \to \infty} e^{\beta t} \mathcal{M}(\mathbf{H})(t) = 0$ ($\lim_{t \to \infty} t^\beta \mathcal{M}(\mathbf{H})(t) = 0$).
\end{definition}
Similar to \citet{wang2023.InverseApproximationTheory}, this adjusted memory function definition is also compatible with the memory concept in linear functional which is based on the famous Riesz representation theorem (\cref{thm:riesz_representation_theorem} in \cref{sec:theoretical_backgrounds}).
In the linear functional case, this memory function is the impulse response function. It measures the decay speed of the memory about an impulse given at $t=0$. 
It is a surrogate to characterize the model's memorization about the previous inputs in the hidden states $h_t$ and outputs $y_t$. 
While a large memory value $\mathcal{M}(t)$ does not mean the model at time $t$ has a clear memorization about previous inputs $x_0$, a small memory value $\mathcal{M}(t)$ means the model has forgotten the impulse input $x_0$. 
\textbf{Therefore, having a slow decay memory function $\mathcal{M}(\cdot)$ is a necessary condition to build a model with long-term memory.}
As shown in \cref{subsec:point-wise_continuity}, the nonlinear functionals constructed by state-space models are point-wise continuous over Heaviside inputs. 
Combined with time-homogeneity, we know that state-space models are nonlinear functionals with decaying memory (see \cref{subsec:point-wise-continuity-to-decaying-memory}).

\begin{definition}[Functional sequence approximation in Sobolev-type norm]
Given functional sequences $\mathbf{H}$ and $\widehat{\mathbf{H}}$, we consider the approximation in the following Sobolev-type norm~(\cref{subsec:sobolev_norm}):
\begin{align}
\label{eq:Sobolev_norm}
    & \left \|\mathbf{H} - \widehat{\mathbf{H}} \right \|_{W^{1, \infty}} := \\
    & \sup_t \left (\|H_t - \widehat{H}_t\|_\infty + \left \|\frac{dH_t}{dt} - \frac{d\widehat{H}_t}{dt} \right \|_\infty \right ).
\end{align}
\end{definition}

\begin{definition}[Perturbation error]
    For target $\mathbf{H}$ and parameterized model $\widehat{\mathbf{H}}(\cdot, \theta_m), \theta_m = (\Lambda, U, b, c) \in \Theta_m:=\{\mathbb{R}^{m \times m} \times \mathbb{R}^{m \times d} \times \mathbb{R}^{m} \times \mathbb{R}^m \}$, we define the \emph{perturbation error} for hidden dimension $m$:
    \begin{equation}
    \label{eq:perturbed_error}
        E_m (\beta) :=
        \sup_{
            \tilde{\theta}_m \in
            \{\theta : |\theta - \theta_m|_2 \leq \beta \}
        }
        \|\mathbf{H} - \widehat{\mathbf{H}}(\cdot; \tilde{\theta}_m)\|_{W^{1, \infty}}.
    \end{equation}
    In particular, $\widetilde{\mathbf{H}}$ refers to the perturbed models $\widehat{\mathbf{H}}(\cdot; \tilde{\theta}_m)$. 
    Moreover, $E(\beta) := \limsup_{m \to \infty} E_m (\beta)$ is the asymptotic perturbation error. 
    The weight norm for SSM is $|\theta|_2 := \max(|\Lambda|_2, |U|_2, |b|_2, |c|_2)$. 
\end{definition}

Based on the definition of perturbation error, we consider the stable approximation as introduced by \citet{wang2023.InverseApproximationTheory}. 

\begin{definition}[Stable approximation]
    Let $\beta_0>0$.
    A target functional sequence $\mathbf{H}$ admits a \emph{$\beta_0$-stable approximation} if the perturbed error satisfies that: 
    \begin{enumerate}
        \item $E(0)=0$. 
        \item $E(\beta)$ is continuous for $\beta \in [0 , \beta_0]$.
    \end{enumerate}
\end{definition}
Equation $E(0)=0$ means the universal approximation is achieved by the hypothesis space.
Stable approximation strengthens the universal approximation by requiring the model to be robust against perturbation on the weights. 
As the stable approximation is the necessary requirement for the optimal parameters to be found by the gradient-based optimizations, it is a desirable assumption.

The ``curse of memory'' phenomenon, which was originally formulated for linear functionals and linear RNNs, is well-documented in prior research~\citep{li2020.CurseMemoryRecurrent,li2022.ApproximationOptimizationTheory,jiang2023.BriefSurveyApproximationa}. 
It describes the phenomenon where targets approximated by linear, hardtanh, or tanh RNNs must demonstrate an exponential decaying memory. 
However, empirical observations suggest that state-space models, particularly the S4 variant, may possess favorable properties. 
Thus, it is crucial to ascertain whether the inherent limitations of RNNs can be circumvented using state-space models.
Given the impressive performance of state-space models, notably S4, a few pivotal questions arise: Do the model structure of state-space models overcome the ``curse of memory''?
In the subsequent section, we will demonstrate that the model structure of state-space models does not indeed address the curse of memory phenomenon.

\section{Main results}
\label{main_results}

In this section, we first prove that similar to the traditional recurrent neural networks~\citep{li2020.CurseMemoryRecurrent, wang2023.InverseApproximationTheory}, state-space models without reparameterization suffer from the ``curse of memory'' problem.
This implies the targets that can be stably approximated by SSMs must have exponential decaying memory. 
Our analysis reveals that the problem arises from recurrent weights converging to a stability boundary when learning targets associated with long-term memory.
Therefore, we introduce a class of stable reparameterization techniques to achieve the stable approximation for targets with polynomial decaying memory.

Beside the benefit of approximation perspective, we also discuss the optimization benefit of the stable reparameterizations. 
We show that the stable reparameterization can make the gradient scale more balanced, therefore the optimization of large models can be more stable. 

\subsection{Curse of memory in SSMs}

\begin{figure*}[t!]{
    \centering
    \subfigure[][SSM]{
        \includegraphics[width=0.315\textwidth]{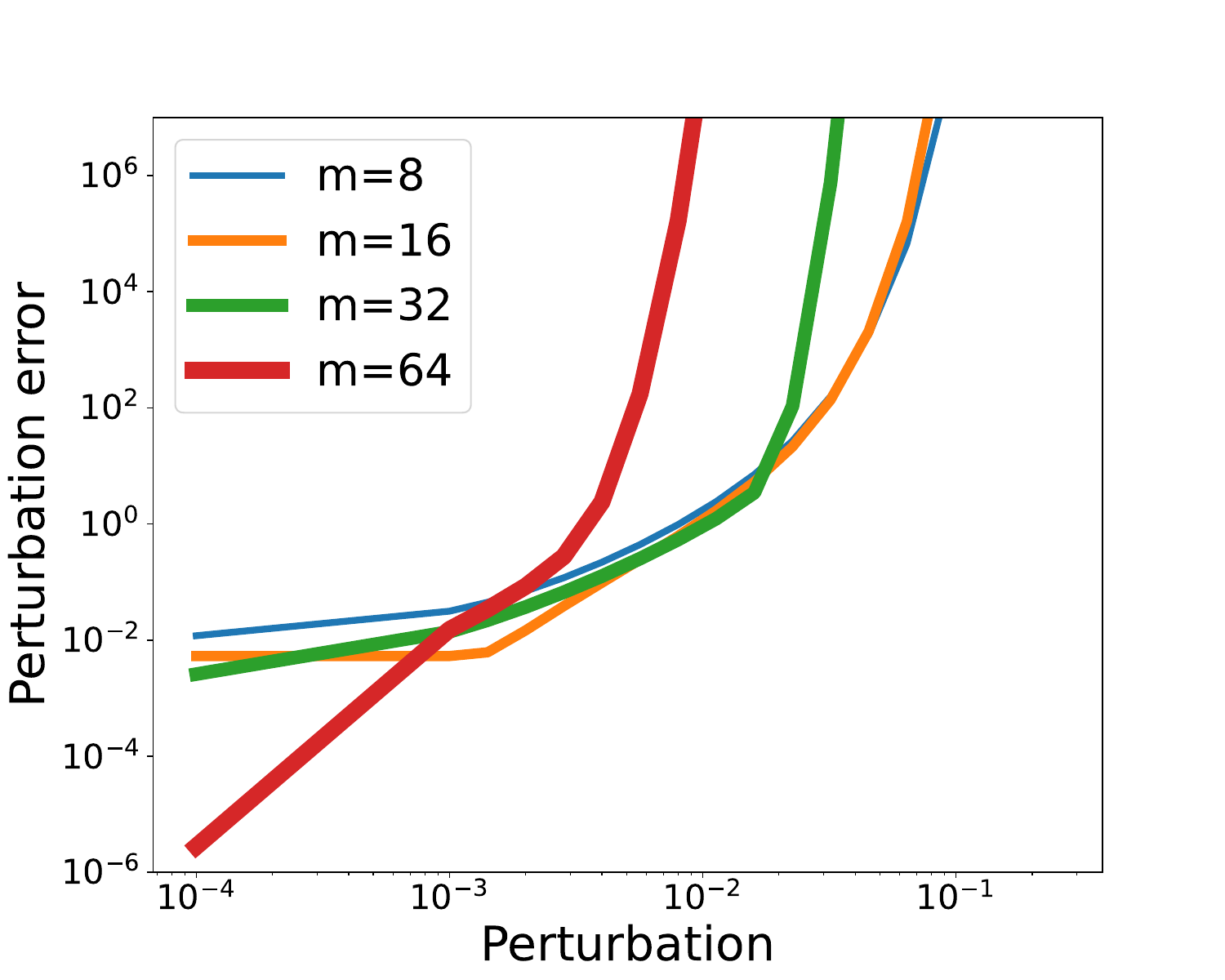}
    }
    \subfigure[][SoftplusSSM]{
        \includegraphics[width=0.315\textwidth]{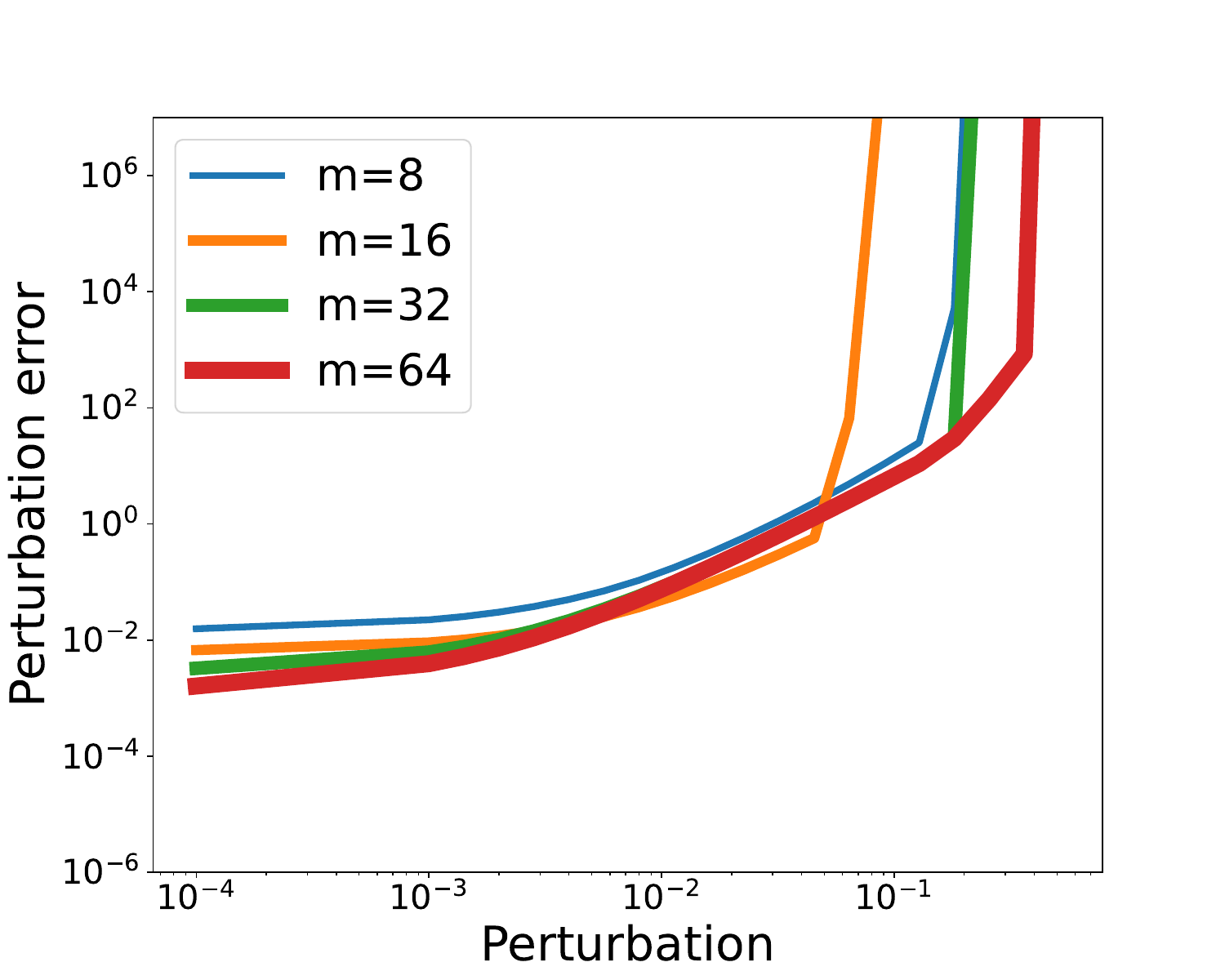}
    }
    \subfigure[][S4]{
        \includegraphics[width=0.315\textwidth]{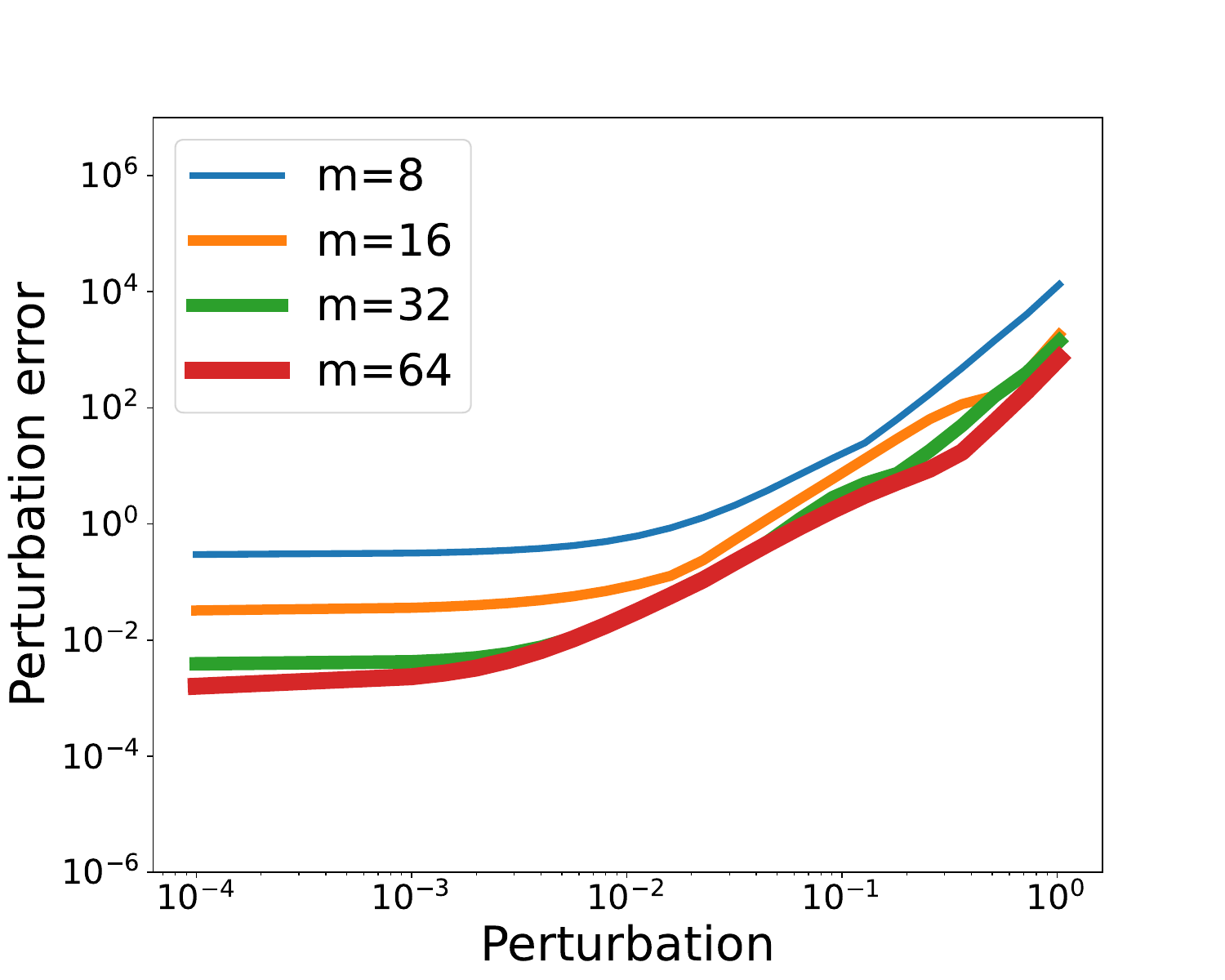}
    }
    \caption{
        State-space models without stable reparameterization cannot approximate targets with polynomial decaying memory. 
        In (a), the intersection of lines are shifting towards left as the hidden dimension $m$ increases.
        In (b), SSMs using softplus reparameterization has a stable approximation.
        In (c), S4 can stably approximate the target with better stability. 
    }
    \label{fig:CoM_SSM}
}
\end{figure*}

\begin{table*}[htb!]
    \caption{
    Impact of stable reparameterizations in approximation and stable approximation. 
    As the reparameterization does not change the hypothesis space of SSMs, both vanilla SSMs and StableSSM are universal approximators. 
    Vanilla SSMs can only stably approximate targets with exponential decay while StableSSM can stably approximate any targets with decaying memory. 
    }
    \label{table:comparison_of_reparameterization_in_approximation}
    \centering
    \begin{tabular}{c|cc}
    \toprule
             & Approximation & Stable approximation    \\
    \midrule
    Without reparameterization (Vanilla SSM)        & Universal~\citep{wang2023.StatespaceModelsLayerwisea} & Not universal (Thm \ref{thm:com_ssm}) \\
    With stable reparameterization (StableSSM)    & Universal~\citep{wang2023.StatespaceModelsLayerwisea} & Universal (Thm \ref{thm:stable_reparameterization_linear_functional}) \\
    \bottomrule
    \end{tabular}
\end{table*}

In this section, we present a theoretical theorem demonstrating that the state-space structure does not alleviate the ``curse of memory'' phenomenon. 
State-space models consist of alternately stacked linear RNNs and nonlinear activations. 
Our result is established for both the shallow case and deep case~(\cref{remark:generalization_to_multilayer}).
As recurrent models, SSMs without reparameterization continue to exhibit the commonly observed phenomenon of exponential memory decay, as evidenced by empirical findings~\citep{wang2023.StatespaceModelsLayerwisea}.

\begin{assumption}
We assume the hidden states remain uniformly bounded for any input sequence $\mathbf{x}$, irrespective of the hidden dimensions $m$.
Specifically, this can be expressed as
\begin{equation}
\label{eq:uniformly_bounded_hiddens}
    \sup_m \sup_{t} |h_t|_{\infty} < \infty.  
\end{equation}
\end{assumption}

\begin{assumption}
We focus on strictly increasing, continuously differentiable nonlinear activations with Lipschitz constant $L_0$.
This property holds for activations such as tanh, sigmoid, softsign $\sigma(z) = \frac{z}{1+|z|}$. 
\end{assumption}

\begin{theorem}[Curse of memory in SSMs]
    \label{thm:com_ssm}
    Assume $\mathbf{H}$ is a sequence of bounded, causal, continuous, regular and time-homogeneous functionals on $\mathcal{X}$ with decaying memory.
    Suppose there exists a sequence of state-space models $\{\widehat{\mathbf{H}}(\cdot, \theta_m)\}_{m=1}^{\infty}$ $\beta_0$-stably approximating $\mathbf{H}$ in the norm defined in \cref{eq:Sobolev_norm}.
    Assume the model weights are uniformly bounded: $\theta_{\max} := \sup_m |\theta_{m}|_2 < \infty$.
    Then the memory function $\mathcal{M}(\mathbf{H})(t)$ of the target decays exponentially:
    \begin{equation}
        \mathcal{M}(\mathbf{H})(t) \leq (d+1) L_{0} \theta_{\max}^2 e^{-\beta t}, \quad t \geq 0, \beta < \beta_0.
    \end{equation}
    Here $d$ is the dimension of input sequences. 
    When generalized to multi-layer cases, the memory function bound induced from $\ell$-layer SSM is: For some polynomial $P(t)$ with degree at most $l-1$
    \begin{equation}
        \mathcal{M}(\mathbf{H})(t) \leq (d+1) L_{0}^{\ell} \theta_{\max}^{\ell+1} P(t) e^{-\beta t}, \quad t \geq 0, \beta < \beta_0.
    \end{equation}
\end{theorem}
The proof of \cref{thm:com_ssm} is provided in \cref{subsec:proof_for_com}.
The (continuous-time) stability boundary (discussed in \cref{remark:stability_boundary}) for $\Lambda$ in state-space models (\cref{eq:2}) is $\max_{i \in [m]}\lambda_i(\Lambda) < 0$.
This boundary comes from the stabiltiy criterion for linear time-invariant system.
Compared with previous results~\citep{li2020.CurseMemoryRecurrent,wang2023.InverseApproximationTheory}, the main proof difference comes from \cref{lemma:decaying_v} as the activation is in the readout $y_t = c^\top \sigma(h_t)$. 
Our results provide a more accurate characterization of memory decay, in contrast to previous works that only offer qualitative estimates.
A consequence of \cref{thm:com_ssm} is that if the target exhibits a non-exponential decay (e.g., polynomial decay), the recurrent weights converge to a stability boundary, thereby making the approximation unstable.
Finding optimal weights can become challenging with gradient-based optimization methods, as the optimization process tends to become unstable with the increase of model size.
The numerical verification is presented in \cref{fig:CoM_SSM} (a).
The lines intersect and the intersections points shift towards the 0, suggesting that the stable radius $\beta_0$ does not exist.
Therefore SSMs without reparameterization cannot stably approximate targets with polynomial decaying memory.

\subsection{Stable reparameterization and its advantage in approximation}

The proof of \cref{thm:com_ssm} suggests that the ``curse of memory'' arises due to the recurrent weights approaching a stability boundary. 
Additionally, our numerical experiments (in \cref{fig:CoM_SSM} (c)) show that while state-space models suffer from curse of memory, the commonly used S4 layer (with exponential reparameterization) ameliorates this issue. 
However, it is not a unique solution.
Our findings highlight that the foundation to achieving a stable approximation is the stable reparameterization method, which we define as follows:

\begin{definition}[Stable reparameterization]
\label{def:stable_reparameterization}
    We say a reparameterization scheme $f: \mathbb{R} \to \mathbb{R}$ is stable if there exists a continuous function $g$ such that: $g: [0, \infty) \to [0, \infty), g(0)=0$:
    \begin{equation}
    \label{eq:stable_reparameterization}
        \sup_w \left[ |f(w)| \sup_{|\tilde{w} - w| \leq \beta} \int_0^\infty \left | e^{f(\tilde{w}) t}  - e^{f(w) t} \right | dt \right ] \leq g(\beta).
    \end{equation}
    For example, commonly used reparameterization~\citep{gu2022.EfficientlyModelingLonga,smith2023.SimplifiedStateSpace} such as $f(w)=-e^{w}$, $f(w)=-\log(1+e^{w})$ are all stable. 
    Verifications are provided in \cref{remark:verification_of_stable_reparameterization}.
\end{definition}

As depicted in \cref{fig:CoM_SSM} (b), state-space models with stable reparameterization can approximate targets exhibiting polynomial decay in memory.
In particular, we prove that under a simplified perturbation setting (solely perturbing the recurrent weights), any linear functional can be stably approximated by linear RNNs.
This finding under simplified setting is already significant as the instability in learning long-term memory mainly comes from the recurrent weights.

\begin{theorem}[Existence of stable approximation by stable reparameterization]
\label{thm:stable_reparameterization_linear_functional}
    For \textbf{any} bounded, causal, continuous, regular, time-homogeneous linear functional $\mathbf{H}$, assume $\mathbf{H}$ is approximated by a sequence of linear RNNs $\{\widehat{\mathbf{H}}(\cdot, \theta_m)\}_{m=1}^{\infty}$ with stable reparameterization, then this approximation is a stable approximation.
\end{theorem}
The proof of \cref{thm:stable_reparameterization_linear_functional} is in \cref{subsec:proof_for_stable_reparameterization_linear_functional}. 
The generalization to nonlinear functionals with Volterra-Series representation can be similarly achieved (\cref{remark:existence_of_stable_approximation_for_nonlinear_functionals}). 
Compared to \cref{thm:com_ssm}, \cref{thm:stable_reparameterization_linear_functional} underscores the role of stable reparameterization in achieving stable approximation of nonlinear functional with long-term memory.
Although vanilla SSM and StableSSM operate within the same hypothesis space, StableSSM demonstrates better stability in approximating any decaying memory target (\cref{table:comparison_of_reparameterization_in_approximation}). 
In contrast, the vanilla SSM model is limited to stably approximate targets characterized by an exponential memory decay.

\begin{figure*}[ht!]{
    \centering
    \includegraphics[width=0.68\textwidth]{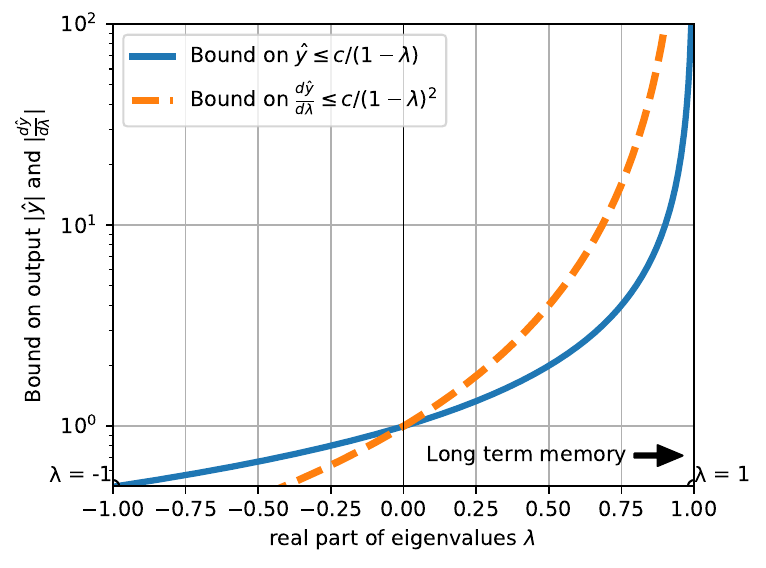}
    \caption{
    The scaling of layer output bound $|\hat{y}| \leq \frac{c}{1-\lambda}$ and the gradients $|\frac{d \hat{y}}{d\lambda}| \leq \frac{c}{(1-\lambda)^2}$. 
    The stability boundary is $\lambda=\pm 1$. 
    When the model adapts to learn long-term memory (as $\lambda$ approaches 1), the gradient experiences an increase that surpasses the rate of output growth. 
    Techniques like layer normalization are insufficient to address this issue of exploding gradients effectively.
    }
    \label{fig:grad_norm_stable_parameterization_v2}
}
\end{figure*}

\subsection{Optimization benefit of stable reparameterization}

In the previous section, the approximation benefit of stable reparameterizations in SSMs is discussed. 
Here we study the impact of different parameterizations on the optimization stability, in particular, the gradient scales. 

As pointed out by \citet{li2020.CurseMemoryRecurrent,li2022.ApproximationOptimizationTheory}, the approximation of linear functionals using linear RNNs can be reduced into the approximation of $L_1$-integrable memory function $\rho(t)$ via functions of the form $\hat{\rho}(t) = \sum_{i=1}^m c_i e^{-\lambda_i t}$. 
\begin{equation}
    \rho(t) \approx \sum_{i=1}^m c_i e^{-\lambda_i t}, \quad \lambda_i > 0. 
\end{equation}
Within this framework, $\lambda_i$ is interpreted as the decay mode.
Approaching this from the gradient-based optimization standpoint, and given that learning rates are shared across different decay modes, a fitting characterization for ``good parameterization'' emerges:
\textit{The gradient scale across different memory decays modes should be Lipschitz continuous with respect to the weights scale.}
\begin{equation}
    \label{eq:Good_parameterization}
    |\textrm{Gradient}|:= \left | \frac{\partial \textrm{Loss}}{\partial \lambda_i} \right | \leq L |\lambda_i|. 
\end{equation}
The Lipschitz constant is denoted by $L$. 
Without this property, the optimization process can be sensitive to the learning rate. 
We give a detailed discussion in \cref{sec:motivation_for_gow}.
In the following theorem, we first characterize the relationship between gradient norms and recurrent weight parameterization. 

\begin{theorem}[Parameterizations influence the gradient norm scale]
\label{thm:stable_parameterization_opt}
    Assume the target functional sequence $\mathbf{H}$ is being approximated by a sequence of SSMs $\widehat{\mathbf{H}}_m$. 
    If the (diagonal) recurrent weight matrix is parameterized via $f: \mathbb{R} \to \mathbb{R}: f(w) = \lambda$. 
    $w$ is the trainable weight while $\lambda$ is the eigenvalue of recurrent weight matrix $\Lambda$. 
    The gradient norm $G_f(w)$ of weight $w$ is upper bounded by the following function:
    \begin{equation}
    \label{eq:gradient_bound}
        G_{f}(w) := 
        \left | 
            \frac{\partial \textrm{Loss}}{\partial w} 
        \right | 
        \leq C_{\mathbf{H}, \widehat{\mathbf{H}}_m} \frac{|f'(w)|}{f(w)^2}.
    \end{equation}
    Here $C_{\mathbf{H}, \widehat{\mathbf{H}}_m}$ is independent of the parameterization $f$ provided that $\mathbf{H}, \widehat{\mathbf{H}}_m$ are fixed. 
    The discrete-time version is
    \begin{equation}
        G^D_{f}(w) := \left | \frac{\partial \textrm{Loss}}{\partial w} \right | \leq C_{\mathbf{H}, \widehat{\mathbf{H}}_m} \frac{|f'(w)|}{(1-f(w))^2}.
    \end{equation}
\end{theorem}
Refer to \cref{subsec:stable_parameterization_opt} for the proof of Theorem \ref{thm:stable_parameterization_opt}.
In \cref{subsec:Comparison_of_different_recurrent_weights} we summarize common reparameterization methods and corresponding gradient scale functions. 

\begin{remark}[Generalization to multi-layer models]
    We do not prove the gradient bound result for multi-layer case in the paper, here we discuss the idea to genearlize it: 
    Consider a specific layer in a multi-layer model, without loss of generality we also have the boundedness of result from the previous layer and expected inputs for the next layer. 
    If we take the results from previous layer as the inputs and treat the expected inputs for next layer as the outputs, the gradient of recurrent weights for this layer also observe the same gradient norm bound with form in \cref{eq:gradient_bound}. 
    This comes from the fact that the gradient of the selected layer remains unchanged, regardless of whether the remaining layers are frozen or not. 
\end{remark}

\subsection{On the ``best'' parameterization in stability sense}
\label{sec:ode_solution_best_parameterization}

\begin{figure*}[t!]{
        \centering
        \subfigure[][Linear functionals]{
            \includegraphics[width=0.48\textwidth]{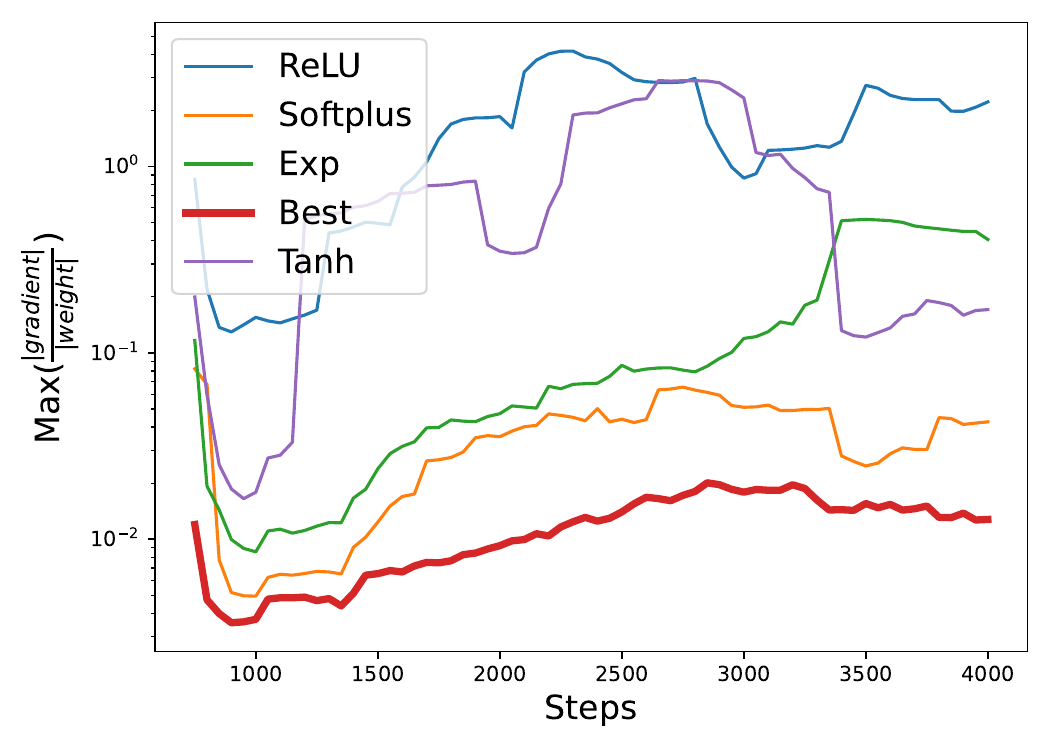}
        }
        \subfigure[][Language model]{
            \includegraphics[width=0.48\textwidth]{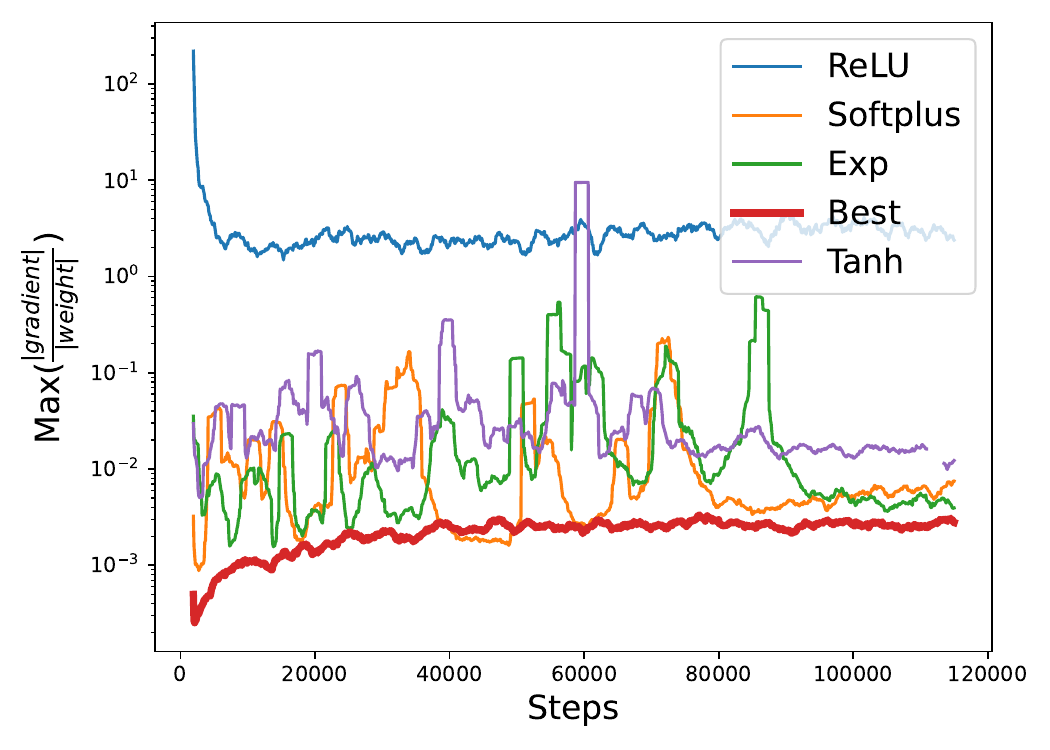}
        }
        \caption{
            In panel (a), in the learning of linear functionals of polynomial decaying memory, the gradient-over-weight scale range during the training of state-space models. 
            It can be seen the ``best''discrete parameterization $f(w) = 1 - \frac{1}{w^2 + 0.5}$ achieves the smallest gradient-over-weight scale. 
            Such property is desirable when a large learning rate is used in training. 
            The ``best'' reparameterization $f(w)=1-\frac{1}{w^2+0.5}$ maintains the smallest $\max(\frac{|\textrm{grad}|}{|\textrm{weight}|})$ which is crucial for the training stability. 
            Similar results can be observed in the language modelling task as in panel (b). 
        }
        \label{fig:loss_and_grad_norm_stable_parameterization_in_LF}
        }
    \end{figure*}
       
According to the criterion given in \cref{eq:Good_parameterization}, the ``best'' stable reparameterization should satisfy the following equation for some constant $L > 0$.
\begin{equation}
\label{eq:ode_for_best_parameterization}
    G_f(w) \leq C_{\mathbf{H}, \widehat{\mathbf{H}}_m} \frac{|f'(w)|}{f(w)^2} = L |w|.
\end{equation}
Based on the criterion, a sufficient condition for the above criterion is to find some function $f$ that satisfies the following equation for some real $a, b \in \mathbb{R}$:
\begin{align}
    \frac{f'(w)}{f(w)^2} & = \frac{d(-\frac{1}{f(w)})}{dw} = 2aw,\\
    \frac{1}{f(w)} & = -(a w^2 + b) \\
    \Rightarrow f(w) & = -\frac{1}{a w^2 + b}.
\end{align}
The first equation is achieved by integrating the function $\frac{f'(w)}{f(w)^2}$. 
Therefore the ``best'' parameterization under the assumption of the Lipschitz property of gradient is characterized by the function with two degrees of freedom:
By stability requirement $f(w) \leq 0$ for all $w$
\begin{equation}
    f(w) = -\frac{1}{aw^2 + b}, \quad a > 0, b \geq 0.
\end{equation}
Similarly, the discrete case gives the solution
$f(w) = 1-\frac{1}{a w^2 + b}.$
The stability of linear RNN further requires $a > 0$ and $b \geq 0$. 
We choose $a=1, b=0.5$ because this ensures the stability of the hidden state dynamics and stable approximation in \cref{eq:stable_reparameterization}. 
Notice that $\lim_{w \to 0} 1 - \frac{1}{w^2 + 0.5} = -1$ which does not cross the stability boundary $\lambda=-1$.
It can be seen in \cref{fig:grad_norm_stable_parameterization} that, compared with direct and exponential reparameterizations, the softplus reparameterization is generally milder in this gradient-over-weight criterion.
The ``best'' parameterization is optimal in the sense it has a bounded gradient-over-weight ratio across different weights $w$ (different eigenvalues $\lambda$).

\begin{remark}
    Apart from the reparameterization method, a simple yet effective method is gradient clipping. 
    However, clipped gradient is biased there the training effectiveness of the gradient descent might be reduced. 
    In contrast, the reparameterization is changing the scale of the gradient descent by introducing pre-conditioning term $\frac{f'(w)}{f(w)^2}$. 
\end{remark}

\begin{figure*}[bht!]{
    \centering
    \subfigure[][Gradient-weight-ratio at initialization]{
    \includegraphics[width=0.48\textwidth]{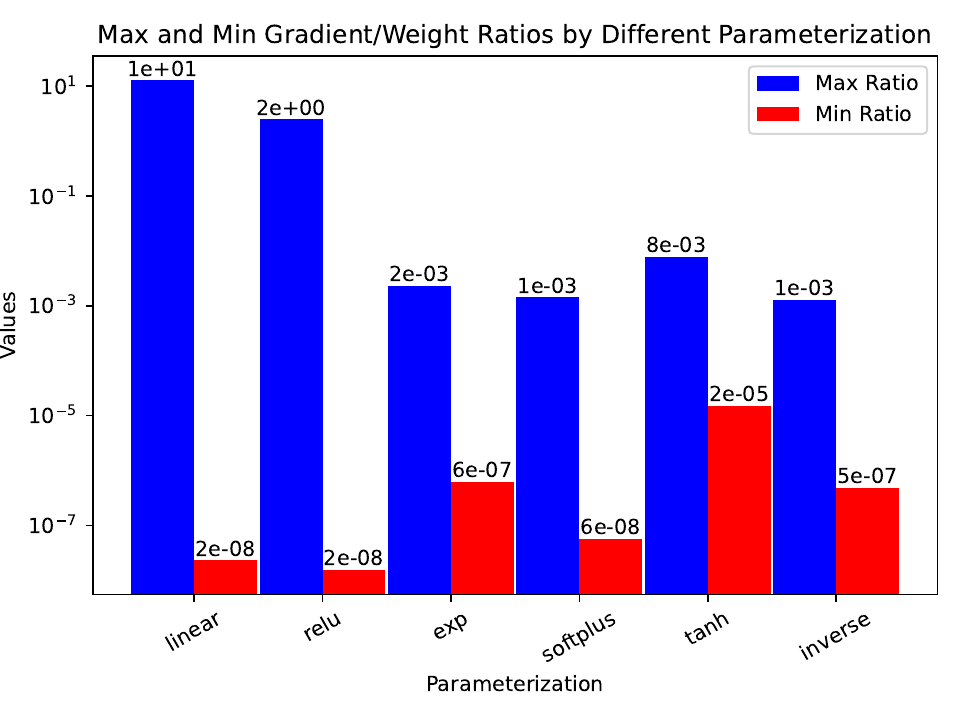}
    }
    \subfigure[][Training loss]{
    \includegraphics[width=0.48\textwidth]{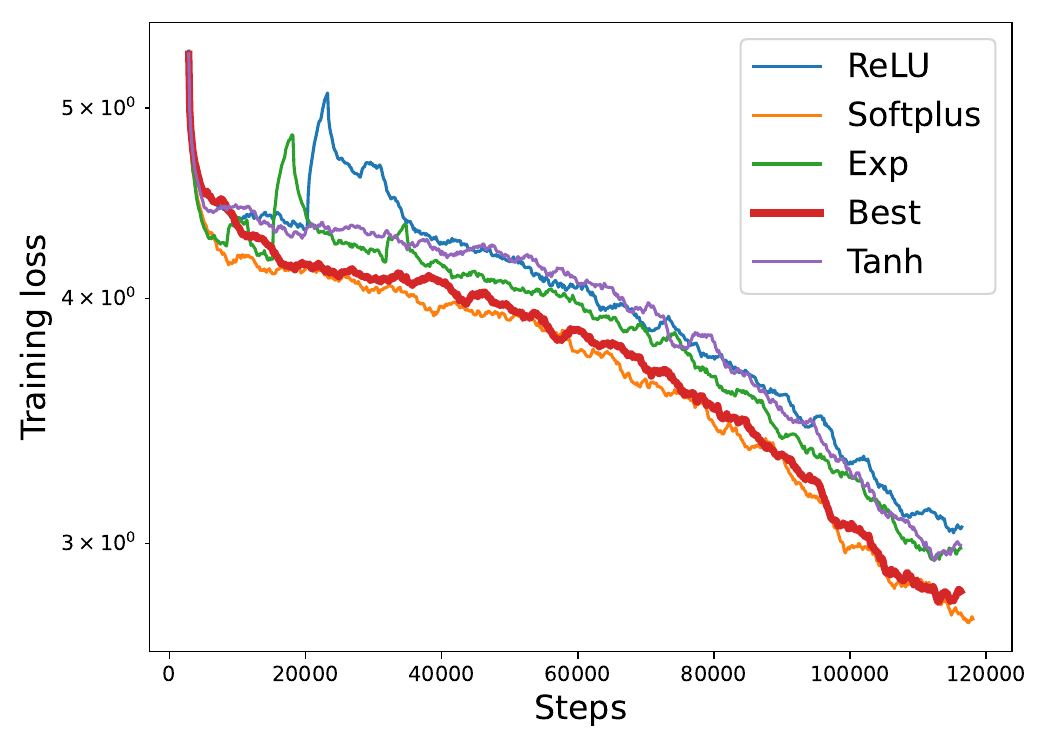}
    }
    \caption{
        Language models on WikiText-103. 
        In the left panel (a), we show the gradient-over-weight ratio ranges for different parameterizations of recurrent weights in state-space models. 
        The eigenvalues $\lambda$ are initialized to be the same while the only difference is the reparameterization function $f$.
        In the right panel (b), the ``Best'' parameterization is more stable than the ReLU and exponential reparameterizations. 
        Additional experiments for different learning rates are provided in \cref{fig:1LR2LR5LR}. 
        }
    \label{fig:grad_norm_stable_parameterization_language_model}
    }
\end{figure*}

\section{Numerical verifications}
\label{nunmerical_verifications}

Based on the above analyses, we verify the theoretical statements over synthetic tasks and language models using WikiText-103. 
The additional numerical details are provided in \cref{appendix:numerical_details}. 

\subsection{Synthetic tasks}
\label{subsec:synthetic_tasks}

Linear functionals have a clear structure, allowing us to study the differences of parameterizations. 
Similar to \citet{li2020.CurseMemoryRecurrent} and \citet{wang2023.InverseApproximationTheory}, we consider linear functional targets $\mathbf{H}$ with following polynomial memory function $\rho(t) = \frac{1}{(t+1)^{1.1}}$: 
$y_t = H_t(\mathbf{x}) = \int_{-\infty}^t \rho(t-s) x_s ds.$
We use the state-space models with tanh activations to learn the sequence relationships. 
In \cref{fig:loss_and_grad_norm_stable_parameterization_in_LF} (a), the eigenvalues $\lambda$ are initialized to be the same while the only difference is the reparameterization function $f(w)$.
Training loss across different reparameterization schemes are similar but the gradient-over-weight ratio across different parameterization schemes are different in terms of the scale.

\begin{table*}[ht!]
    \centering
    \caption{
    Comparison of stability of different parameterizations over MNIST. 
    The experiments conducted on the MNIST and CIFAR10 datasets were replicated three times, with the standard deviation of the test loss indicated in parentheses.}
    \begin{tabular}{|c|c|c|c|c|}
    \hline
    LR        & Direct                & Softplus               & Exp                   & Best                 \\ \hline
    5e-6      & 2.314384 (7.19932e-05)& 2.241642 (0.001279)    & 2.241486 (0.001286)   & \textbf{2.241217} (0.001297)   \\ \hline
    5e-5      & 2.304331 (2.11817e-07)& 0.779663 (0.001801)    & 0.774661 (0.001685)   & \textbf{0.765220} (0.001352)   \\ \hline
    5e-4      & 2.303190 (1.66387e-06)& 0.094411 (0.000028)    & 0.093418 (0.000024)   & \textbf{0.091924} (0.000019)   \\ \hline
    5e-3      & NaN                   & 0.023795 (0.000004)    & 0.023820 (0.000003)   & \textbf{0.023475} (0.000002)   \\ \hline
    5e-2      & NaN                   & 0.802772 (1.69448)     & 0.868350 (1.55032)    & \textbf{0.089073} (0.000774)   \\ \hline
    5e-1      & NaN                   & 2.313510 (0.000014)    & 2.314244 (0.000025)   & \textbf{2.185477} (0.048238)   \\ \hline
    5e+0      & NaN                   & NaN                    & NaN                   & \textbf{199.013813} (50690.6)  \\ \hline
    \end{tabular}
    \label{tab:mnist_train_loss}
\end{table*}

\subsection{Language models}
\label{subsec:language_models}

\begin{table*}[ht!]
    \caption{Comparison of stability of different parameterizations over CIFAR10}
    \centering
    \begin{tabular}{|c|c|c|c|c|}
    \hline
    LR        & Direct               & Softplus              & Exp                  & Best                 \\ \hline
    5e-6      & NaN                  & 1.745752 (0.000006)   & 1.745816 (0.000009)  & \textbf{1.745290} (0.000011)  \\ \hline
    5e-5      & NaN                  & 1.220859 (0.000008)   & 1.218064 (0.000008)  & \textbf{1.215510} (0.000014)  \\ \hline
    5e-4      & NaN                  & 0.883649 (0.000898)   & \textbf{0.866817} (0.000328)  & 0.870412 (0.000442)  \\ \hline
    5e-3      & NaN                  & 1.449352 (0.000414)   & 1.567662 (0.021489)  & \textbf{1.364697} (0.013849)  \\ \hline
    5e-2      & NaN                  & 1.942372 (0.011317)   & 1.846173 (0.007990)  & \textbf{1.713892} (0.013426)  \\ \hline
    5e-1      & NaN                  & 37.802437 (3776.6383) & \textbf{2.296230} (0.000984)  & 2.554265 (0.168649)  \\ \hline
    5e+0      & NaN                  & \textbf{540.621033} (NaN)      & NaN                  & 615.374522 (30795.4) \\ \hline
    \end{tabular}
    \label{tab:cifar10_train_loss}
\end{table*}

In addition to the synthetic dataset of linear functionals, we further justify \cref{thm:stable_parameterization_opt} by examining the gradient-over-weight ratios for language models using state-space models (S5). 
In particular, we adopt the Hyena~\citep{poli2023.HyenaHierarchyLarger} architecture while the implicit convolution is replaced by a simple real-weighted state-space model~\citep{smith2023.SimplifiedStateSpace}. 

In \cref{fig:grad_norm_stable_parameterization_language_model} (a), given the same initialization, we show that stable reparameterizations such as exponential, softplus, tanh and ``best'' exhibit a narrower range of gradient-over-weight ratios compared to both the direct and relu reparameterizations.
Beyond the gradient at the same initialization, in \cref{fig:loss_and_grad_norm_stable_parameterization_in_LF} (b), we show the gradient-over-weight ratios during the training process. 
The stable reparameterization will give better gradient-over-weight ratios in the sense that the ``best'' stable reparameterization maintains the smallest $\max (\frac{|\textrm{grad}|}{|\textrm{weight}|})$. 
Specifically, as illustrated in \cref{fig:grad_norm_stable_parameterization_language_model} (b) and \cref{fig:1LR2LR5LR}, while training with a large learning rate may render the exponential parameterization unstable, the ``best'' reparameterization $f(w)=1-\frac{1}{w^2+0.5}$ appears to enhance training stability.

\begin{table*}[ht!]
    \centering
    \begin{tabular}{c|cccccc|c}
        & Listops	& Text	& Retrieval	& Image	& Pathfinder	& Pathx	& Avg  \\
    \hline
    Exp parameterization (S4) & 59.60	& 86.82	& 90.90	& \textbf{88.65}	& 94.2	& \textbf{96.35}	& 86.09 \\
    Best parameterization     & \textbf{60.80}	& \textbf{88.5}	& \textbf{91.3}	& 87.39	& \textbf{94.8}	& 96.1	& \textbf{86.48}
    \end{tabular}
    \caption{Comparison of parameterizations on long range arena.}
    \label{tab:lra}
\end{table*}

\subsection{Image classification}

Apart from the gradient scale range shown in the language modeling experiments, we further compare the stability of different parameterization schemes over different initial learning rates. 
As shown in the following \cref{tab:mnist_train_loss} and \cref{tab:cifar10_train_loss}, we found that the ``best'' parameterization can be trained with a larger learning rates while exp/softplus parameterizations cannot be trained with larger learning rates (lr=5.0). 
Although the models exhibit comparable performance at lower learning rates, the ``best'' parameterization consistently outperforms others across a range of learning rates
As the training stability issue has been widely reported for larger models
\footnote{
    \url{https://github.com/state-spaces/mamba/issues/6}
    }
\footnote{
    \url{https://github.com/state-spaces/mamba/issues/22}
    }
, we believe the improved training stability is an important component in the scale-up large language models. 

\subsection{Long Range Arena}

We further verify the effectiveness of stable parameterization over the long range arena, as shown in \cref{tab:lra}. Both the exponential and best parameterizations demonstrate stability, yet the best parameterization delivers slightly superior average performance across the long range arena (LRA)~\citep{tay2021.LongRangeArena} benchmark.

\section{Related works}
\label{related_works}

\paragraph*{RNN}
RNNs, as introduced by \citet{rumelhart1986.LearningRepresentationsBackpropagating}, represent one of the earliest neural network architectures for modeling sequential relationships. 
Empirical findings by \citet{bengio1994.LearningLongtermDependencies} have shed light on the challenge of exponential decaying memory in RNNs. 
Various works~\citep{hochreiter1997.LongShorttermMemory, rusch2022.CoupledOscillatoryRecurrent,wang2023improve} have been done to improve the memory patterns of recurrent models. 
Theoretical approaches~\citep{li2020.CurseMemoryRecurrent, li2022.ApproximationOptimizationTheory, wang2023.InverseApproximationTheory} have been taken to study the exponential memory decay of RNNs.
In this paper, we study the state-space models which are also recurrent. 
Our findings theoretically justify that although SSMs variants exhibit good numerical performance in long-sequence modeling~\citep{gu2022.EfficientlyModelingLonga}, simple SSMs also suffer from the ``curse of memory''.

\paragraph*{SSM}
State-space models~\citep{siivola2003.StatespaceMethodLanguage}, previously discussed in control theory, has been widely used to study the dynamics of complex systems. 
The subsequent variants, S4\citep{gu2022.EfficientlyModelingLonga}, S5~\citep{smith2023.SimplifiedStateSpace}, RetNet~\citep{sun2023retentive} and Mamba~\citep{gu2023.MambaLinearTimeSequence}, have significantly enhanced empirical performance.
Notably, they excel in the long-range arena~\citep{tay2021.LongRangeArena}, an area where transformers traditionally underperform.
Contrary to the initial presumption, our investigations disclose that the ability to learn long-term memory is not derived from the linear RNN coupled with nonlinear layer-wise activations. 
Rather, our study underscores the benefits of stable reparameterization in both approximation and optimization.

\paragraph*{Fading memory}
This paper studies the targets with decaying memory. A slightly different memory concept (fading memory) has been studied in literature~\citep{boyd1984.AnalyticalFoundationsVolterra,boyd1985.FadingMemoryProblem}. 
A critical difference is: fading memory is defined with respect to a particular weight function while decaying memory is defined without a specific weight function. 
While both concepts are similar in characterizing the speed of target memory decay, they are still distinct. For instance, there are examples with decaying memory but not fading memory (the peak-hold operator introduced in~\citet{boyd1985.FadingMemoryProblem}) and vice versa (examples with fading memory but not decaying memory are detailed in Appendix A.7 in~\citet{wang2023.InverseApproximationTheory}).

\section{Conclusion}
\label{conclusion}

In this paper, we study the intricacies of long-term memory learning in state-space models, specifically emphasizing the role of recurrent weights parameterization. 
We prove that state-space models without reparameterization fail to stably approximating targets that exhibit non-exponential decaying memory. 
Our analysis indicates this ``curse of memory'' phenomenon is caused by the eigenvalues of recurrent weight matrices converging to stability boundary. 
As an alternative, we introduce a class of stable reparameterization as a robust solution to this challenge, which also partially explains the performance of S4. 
With stable reparameterization, state-space models can stably approximate any targets with decaying memory. 
We also explore the optimization advantages associated with stable reparameterization, especially concerning gradient-over-weight scale. 
Our results give the theoretical support to observed advantages of reparameterizations in S4 and moreover give principled methods to design \textbf{``best'' reparameterization scheme in the optimization stability sense.}
This paper shows that stable reparameterization not only enables the learning of targets with long-term memory but also enhances the optimization stability.

\newpage

\section*{Acknowledgements}

This research is supported by the National Research Foundation, Singapore, under the NRF fellowship (project No. NRF-NRFF13-2021-0005).
Shida Wang is supported by NUS-RMI Scholarship. 

\section*{Impact Statement}

This paper study the approximation and optimization properties of parameterization in state-space models. 
This paper presents work whose goal is to advance the field of Machine Learning. 
There are minor potential societal consequences of our work, none which we feel must be specifically highlighted here.

\bibliography{example_paper}
\bibliographystyle{icml2024}

\newpage
\appendix
\onecolumn

\section{Graphical demonstration of state-space models as stack of \cref{eq:ssm}}
\label{sec:graphical_demonstration_eq2}
Here we show that \cref{eq:ssm} corresponds to the practical instantiation of SSM-based models in the following sense: 
As shown in \cref{fig:NLSSM2}, any practical instantiation of SSM-based models can be implemented as a stack of \cref{eq:ssm}. 
The pointwise shallow MLP can be realized with two-layer state-space models with layer-wise nonlinearity by setting recurrent weights $W$ to be 0. 
\begin{figure}[ht!]
    \centering
    \includegraphics[width=0.95\linewidth]{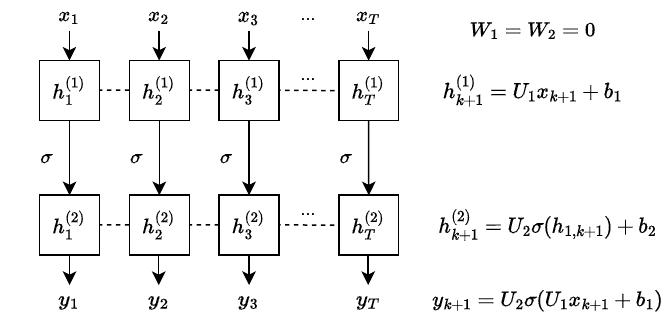}
    \caption{MLP can be realized by two-layer state-space models. The superscript indicates the layers while the subscript indicates the time index. It can be seen MLP is equivalent to SSMs having zero recurrent weights $W_1=W_2=0$.}
    \label{fig:NLSSM2}
\end{figure}

\section{Theoretical backgrounds}
\label{sec:theoretical_backgrounds}
In this section, we collect the definitions for the theoretical statements.

\subsection{Properties of targets}
\label{subsec:properties_of_targets}
We first introduce the definitions on (sequences of) functionals as discussed in~\citep{wang2023.InverseApproximationTheory}.
\begin{definition}{}
    \label{def:functional_properties}
    Let $\mathbf{H} = \{H_t: \mathcal{X} \mapsto \mathbb{R}; t \in \mathbb{R}\}$ be a sequence of functionals.
    \begin{enumerate}
        \item (\textbf{Linear}) $H_t$ is linear functional if for any $\lambda, \lambda' \in \mathbb{R}$ and $\mathbf{x}, \mathbf{x}'\in\mathcal{X}$, $H_t(\lambda \mathbf{x} + \lambda' \mathbf{x}') = \lambda H_t(\mathbf{x}) + \lambda' H_t(\mathbf{x}')$.

        \item (\textbf{Continuous}) $H_t$ is continuous functional if for any $\mathbf{x},' \mathbf{x}\in \mathcal{X}$, $\lim_{{\mathbf{x}' \to \mathbf{x}}} |H_t(\mathbf{x}') - H_t(\mathbf{x})| = 0$.

        \item (\textbf{Bounded}) $H_t$ is bounded functional if the norm of functional $\|H_t\|_{\infty}:=\sup_{\{\mathbf{x} \neq 0\}} \frac{|H_t(\mathbf{x})|}{\|\mathbf{x}\|_{\infty} + 1} + |H_t(\mathbf{0})| < \infty$.

        \item (\textbf{Time-homogeneous}) $\mathbf{H} = \{H_t: t \in \mathbb{R}\}$ is time-homogeneous (or time-shift-equivariant) if the input-output relationship commutes with time shift:
        let $[S_\tau(\mathbf{x})]_t= x_{t-\tau}$ be a shift operator,
        then $\mathbf{H}(S_\tau \mathbf{x}) = S_\tau \mathbf{H}(\mathbf{x})$.
        
        \item (\textbf{Causal}) $H_t$ is causal functional if it does not depend on future values of the input. That is, if $\mathbf{x}, \mathbf{x}'$ satisfy $x_t = x_t'$ for any $t \leq t_0$, then $H_t(\mathbf{x}) = H_t(\mathbf{x}')$ for any $t \leq t_0$.

        \item (\textbf{Regular}) $H_t$ is regular functional if for any sequence $\{\mathbf{x}^{(n)}: n \in \mathbb{N}\}$ such that $x^{(n)}_s \to 0$ for almost every $s \in \mathbb{R}$, then $\lim_{n \to \infty} H_t(\mathbf{x}^{(n)}) = 0.$
    \end{enumerate}
\end{definition}

\subsection{Approximation in Sobolev norm}
\label{subsec:sobolev_norm}
\begin{definition}
    In sequence modeling as a nonlinear functional approximation problem, we consider the Sobolev norm of the functional sequence defined as follow:
\begin{equation}
    \left \|\mathbf{H} - \widehat{\mathbf{H}} \right \|_{W^{1, \infty}} = \sup_t \left (\|H_t - \widehat{H}_t\|_\infty + \left \|\frac{dH_t}{dt} - \frac{d\widehat{H}_t}{dt} \right \|_\infty \right ).
\end{equation}
\end{definition}
Here $\mathbf{H} = \{H_t: t \in \mathbb{R}\}$ is the target functional sequence to be approximated while the $\widehat{\mathbf{H}} = \{\widehat{H}_t: t \in \mathbb{R}\}$ is the model we use. 

In particular, the nonlinear functional operator norm is given by:
\begin{equation}
\label{eq:Operator_norm}
    \|H_t\|_{\infty} := \sup_{\mathbf{x} \neq \mathbf{0}} \frac{|H_t(\mathbf{x})|}{\|\mathbf{x}\|_{\infty} + 1} + |H(\mathbf{0})|.
\end{equation}
As $\mathbf{H}(\mathbf{0}) = 0$, $\|H_t\|_{\infty}$ is reduced to $\displaystyle \sup_{\mathbf{x} \neq 0} \frac{|H_t(\mathbf{x})|}{\|\mathbf{x}\|_{\infty} + 1}$.
If $\mathbf{H}$ is a linear functional, this definition is compatible with the common linear functional norm in \cref{eq:linear_functional_norm}. 

We check this operator norm in \cref{eq:Operator_norm} is indeed a norm:
Without loss of generality, we will drop the time index for brevity. 
\begin{enumerate}
    \item Triangular inequality: For nonlinear functional $H_1$ and $H_2$,
    \begin{align}
        \|H_1 + H_2\|_{\infty} 
        & := \sup_{\mathbf{x} \neq \mathbf{0}} \frac{|(H_1+H_2)(\mathbf{x})|}{\|\mathbf{x}\|_{\infty} + 1} \\
        & \leq \sup_{\mathbf{x} \neq \mathbf{0}} \frac{|H_1(\mathbf{x})|}{\|\mathbf{x}\|_{\infty} + 1} + \sup_{\mathbf{x} \neq \mathbf{0}} \frac{|H_2(\mathbf{x})|}{\|\mathbf{x}\|_{\infty} + 1} = \|H_1\|_{\infty} + \|H_2\|_{\infty}.
    \end{align}
    The inequality is by the property of supremum. 
    \item Absolute homogeneity: 
    For any real constant $s$ and nonlinear functional $H$
    \begin{align}
        \|s H\|_{\infty} 
        := \sup_{\mathbf{x} \neq \mathbf{0}} \frac{|(s H)(\mathbf{x})|}{\|\mathbf{x}\|_{\infty} + 1} 
        = |s| \sup_{\mathbf{x} \neq \mathbf{0}} \frac{|H(\mathbf{x})|}{\|\mathbf{x}\|_{\infty} + 1}
        = |s| \|H\|_{\infty} .
    \end{align}
    \item Positive definiteness: 
    If $\|H\|_{\infty}  = 0$, then for all non-zero inputs $\mathbf{x} \neq \mathbf{0}$ we have $H(\mathbf{x}) = 0$. As $H(\mathbf{0}) = 0$, then we know $H$ is a zero functional. 
\end{enumerate}

\paragraph{Property of nonlinear functional sequence norm}
The definition of functional product is by the element-wise product: $(\mathbf{H}_1 \mathbf{H}_2) (\mathbf{x}) = \mathbf{H}_1(\mathbf{x}) \odot \mathbf{H}_2 (\mathbf{x})$. 
As the functional norm satisfies:
\begin{align}
    \|H_1 H_2\|_{\infty} 
    & := \sup_{\mathbf{x} \neq 0} \frac{|H_1(\mathbf{x}) H_2(\mathbf{x})|}{\|\mathbf{x}\|_{\infty} + 1} + |H_1(\mathbf{0}) H_2(\mathbf{0})| \\
    & \leq \sup_{\mathbf{x} \neq \mathbf{0}}\frac{|H_1(\mathbf{x})|}{\|\mathbf{x}\|_{\infty} + 1} \frac{|H_2(\mathbf{x})|}{\|\mathbf{x}\|_{\infty} + 1} + |H_1(\mathbf{0})| \cdot |H_2(\mathbf{0})| \\
    & \leq \sup_{\mathbf{x} \neq \mathbf{0}}\left ( \frac{|H_1(\mathbf{x})|}{\|\mathbf{x}\|_{\infty} + 1} + |H_1(\mathbf{0})| \right ) \sup_{\mathbf{x} \neq \mathbf{0}}\left ( \frac{|H_2(\mathbf{x})|}{\|\mathbf{x}\|_{\infty} + 1} + |H_2(\mathbf{0})| \right ) \\
    & = \|H_1\|_{\infty} \|H_2\|_{\infty} 
\end{align}
Therefore we have 
\begin{align}
\label{eq:property_of_nonlinear_functional_sequence_norm}
    \|\mathbf{H}_1 \mathbf{H}_2\|_{\infty} 
    & = \sup_t \left (\|H_1 H_2\|_\infty + \left \|\frac{d(H_1H_2)}{dt} \right \|_\infty \right )  \\
    & = \sup_t \left (\|H_1 H_2\|_\infty + \left \| H_1 \frac{dH_2}{dt} \right \|_\infty + \left \|\frac{dH_1}{dt} H_2 \right \|_\infty \right )  \\
    & \leq \sup_t \left ( \|H_1\|_{\infty} \|H_2\|_{\infty} + \|H_1\|_{\infty} \left \| \frac{dH_2}{dt} \right \|_{\infty} + \left \|\frac{dH_1}{dt} \right \|_{\infty} \|H_2\|_{\infty} \right ) \\
    & \leq \sup_t \left (\|H_1\|_{\infty} + \left \|\frac{dH_1}{dt} \right \|_{\infty}  \right ) \sup_t \left ( \|H_2\|_{\infty} + \left \|\frac{dH_2}{dt} \right \|_{\infty}) \right )\\
    & = \|\mathbf{H}_1\|_{\infty}  \|\mathbf{H}_2\|_{\infty} 
\end{align}

\subsection{Riesz representation theorem for linear functional}
\label{subsec:Riesz_representation_theorem}

\begin{theorem}[Riesz-Markov-Kakutani representation theorem]
\label{thm:riesz_representation_theorem}
    Assume $H : C_0(\mathbb{R}, \mathbb{R}^d) \mapsto \mathbb{R}$ is a linear and continuous functional. 
    Then there exists a unique, vector-valued, regular, countably additive signed measure $\mu$ on $\mathbb{R}$ such that
    \begin{align}
        H(\mathbf{x}) = \int_{\mathbb{R}} x_s^\top d\mu(s)
        = \sum_{i=1}^{d} \int_{\mathbb{R}} x_{s,i} d\mu_i(s).
    \end{align}
    In addition, we have the linear functional norm 
    \begin{equation}
    \label{eq:linear_functional_norm}
        \| H \|_{\infty} := \sup_{\| \mathbf{x} \|_\mathcal{X} \leq 1} | H(\mathbf{x}) | = \|\mu\|_1(\mathbb{R}) := \sum_i |\mu_i|(\mathbb{R}).
    \end{equation}
    In particular, this linear functional norm is compatible with the norm considered for nonlinear functionals in \cref{eq:Operator_norm}. 
\end{theorem}

\section{Proofs for theorems and lemmas}
\label{sec:proofs}

In \cref{subsec:point-wise_continuity}, we show that the nonlinear functionals defined by state-space models are point-wise continuous functionals at Heaviside inputs. 
In \cref{subsec:proof_for_com}, the proof for state-space models' exponential memory decaying memory property is given.
In \cref{subsec:proof_for_stable_reparameterization_linear_functional}, we prove the linear RNN with stable reparameterization can stably approximate any linear functional. The target is no longer limited to have an exponenitally decaying memory. 
The gradient norm estimate of the recurrent layer is included in \cref{subsec:stable_parameterization_opt}.

\subsection{Proof for SSMs are point-wise continuous functionals}
\label{subsec:point-wise_continuity}
    
\begin{proof}

Let $\mathbf{x}$ be any fixed Heaviside input. 
Assume $\displaystyle \lim_{k \to \infty} \|\mathbf{x}_k - \mathbf{x}\|_{\infty} = 0$. 
Let $h_{k, t}$ and $h_t$ be the hidden state for inputs $\mathbf{x}_k$ and $\mathbf{x}$. 
Without loss of generality, assume $t>0$. 
The following $|\cdot|$ refers to $p=\infty$ norm. 

By definition of the hidden states dynamics and triangular inequality, since $\sigma(\cdot)$ is Lipschitz continuous
\begin{align}
    \frac{d |h_{k,t} - h_t|}{dt} 
    & = |\sigma(\Lambda h_{k,t} + Ux_{k,t}) - \sigma(\Lambda h_t + Ux_t)| \\
    & \leq L |\Lambda h_{k,t} + Ux_{k,t} - \Lambda h_t - Ux_t| \\
    & = L |\Lambda (h_{k,t} - h_t) + U(x_{k,t} - x_t)| \\
    & \leq L (|\Lambda| |h_{k,t} - h_t| + |U| |x_{k,t} - x_t|).
\end{align}
Here $L$ is the Lipschitz constant of activation $\sigma$. 
Apply the Grönwall inequality to the above inequality, we have: 
\begin{equation}
    |h_{k,t} - h_t| \leq \int_0^t e^{L|\Lambda| (t-s)} L |U| \ |x_{k,s}-x_s| ds.
\end{equation} 
As the inputs are bounded, by dominated convergence theorem we have right hand side converges to 0 therefore 
\begin{equation}
    \lim_{k \to \infty} |h_{k,t} - h_t| = 0, \quad \forall t.
\end{equation}

Let $y_{k, t}$ and $y_t$ be the outputs for inputs $\mathbf{x}_k$ and $\mathbf{x}$. 
Therefore we show the point-wise convergence of $\frac{dH_t}{dt}$ at $\mathbf{x}$: 
\begin{align}
    \lim_{k \to \infty} \left | \frac{dy_{k,t}}{dt} - \frac{dy_t}{dt} \right |
    & = \lim_{k \to \infty} \left | c^\top(\frac{dh_{k,t}}{dt} - \frac{dh_t}{dt}) \right | \\ 
    & \leq \lim_{k \to \infty} |c| L( |\Lambda| |h_{k,t} - h_t| + |U| |x_{k,t} - x_t|)= 0.
\end{align}
\end{proof}

\subsection{Point-wise continuity leads to decaying memory}
\label{subsec:point-wise-continuity-to-decaying-memory}

Here we give the proof of decaying memory based on the point-wise continuity of $\frac{dH_t}{dt}$ and boundedness and time-homogeneity of $\mathbf{H}$: 
\begin{proof}
    $$\lim_{t \to \infty} \left | \frac{dH_t}{dt} (\mathbf{u}^x) \right | = \lim_{t \to \infty} \left |\frac{dH_0}{dt} (x \cdot \1_{\{s \geq -t\}}) \right | = \left |\frac{dH_0}{dt} (\mathbf{x}) \right | = 0.$$
    The first equation comes from time-homogeneity. 
    The second equation is derived from the point-wise continuity where input $\mathbf{x}$ means constant $x$ for all time $\mathbf{x} = x \cdot \1_{\{s \geq -\infty\}}$. 
    The third equation is based on the boundedness and time-homogeneity as the output over constant input should be finite and constant $H_t(\mathbf{x}) = H_s(\mathbf{x})$ for all $s, t$. 
    Therefore $|\frac{dH_0}{dt} (\mathbf{x})|=0$.
\end{proof}

\subsection{Proof for \texorpdfstring{\cref{thm:com_ssm}}{}}
\label{subsec:proof_for_com}

The main idea of the proof is two-fold. 
First of all, we show that state-space models with strictly monotone activation is decaying memory in \cref{lemma:decaying_v}. 
Next, the idea of analysing the memory functions through a transform from $[0, \infty)$ to $(0, 1]$ is similar to previous works \citep{li2020.CurseMemoryRecurrent,li2022.ApproximationOptimizationTheory,wang2023.InverseApproximationTheory}. 
The remainder of the proof follows a standard approach, as the derivatives of the hidden states follow the rules of linear dynamical systems when Heaviside inputs are considered.

\begin{proof}
    Assume the inputs considered are uniformly bounded by $X_0$:
    \begin{equation}
        \|\mathbf{x}\|_{\infty} < X_0. 
    \end{equation}

    Define the derivative of hidden states for unperturbed model to be $v_{m, t} = \frac{dh_{m, t}}{dt}$.
    Similarly, $\tilde{v}_{m, t}$ is the derivative of hidden states for perturbed models $\tilde{v}_{m,t} = \frac{d\tilde{h}_{m,t}}{dt}$.
    
    Since each perturbed model has a decaying memory and the target functional sequence $\mathbf{H}$ has a stable approximation, by \cref{lemma:decaying_v}, we have
    \begin{equation}
        \lim_{t \to \infty} \tilde{v}_{m, t} = 0, \quad \forall m.
    \end{equation}

    If the inputs are limited to Heaviside inputs, the derivative $\tilde{v}_{m, t}$ satisfies the following dynamics: Notice that the hidden state satisfies $h_t=0, t \in (-\infty, 0]$, 
    \begin{align}
        \frac{d\tilde{v}_{m, t}}{dt} 
        & = \widetilde{\Lambda}_m \tilde{v}_{m, t}, \quad t \geq 0       \\
        \tilde{v}_{m, 0} 
        & =  \widetilde{\Lambda}_m h_{0} + \widetilde{U}_m x_{0} + \tilde{b}_m = \widetilde{U}_m x_{0} + \tilde{b}_m \\
        \Rightarrow \tilde{v}_{m, t} 
        & =  e^{\widetilde{\Lambda}_m t} (\widetilde{U}_m x_{0} + \tilde{b}_m).
    \end{align}
    Notice that the perturbed initial conditions of the $\tilde{v}_{m,t}$ are uniformly (in $m$) bounded: 
    \begin{align}
    \label{eq:bound_of_perturbed_initial_conditions}
        \tilde{V}_0 
        & := \displaystyle \sup_m |\tilde{v}_{m, 0}|_2 \\
        & = \sup_m |\widetilde{U}_m x_{0} + \tilde{b}_m|_2 \\
        & \leq \sup_m |\widetilde{U}_m x_{0} + \tilde{b}_m|_2 \\
        & \leq d X_0 (\sup_m \|U_m\|_2+\beta_0) + \sup_m \|b_m\|_2 + \beta_0 \\ 
        & < \infty
    \end{align}
    Here $d$ is the input sequence dimension. 

    Similarly, the unperturbed initial conditions satisfy:
    \begin{align}
    \label{eq:bound_of_initial_conditions}
        V_0 
        & := \displaystyle \sup_m |\tilde{v}_{m, 0}|_2 \\
        & = \sup_m |U_m x_{0} + b_m|_2 \\
        & \leq \sup_m |U_m x_{0} + b_m|_2 \\
        & \leq d X_0 \sup_m \|U_m\|_2 + \sup_m \|b_m\|_2  \\ 
        & \leq (d X_0 + 1)\theta_{max} \\ 
        & < \infty
    \end{align}
    
    Select a sequence of perturbed recurrent matrices $\{\widetilde{\Lambda}_{m, k}\}_{k=1}^\infty$ satisfying the following two properties:
    \begin{enumerate}
        \item $\widetilde{\Lambda}_{m, k}$ is Hyperbolic, which means the real part of the eigenvalues of the matrix are nonzero.
        \item $\lim_{k \to \infty} (\widetilde{\Lambda}_{m, k} - \Lambda_m) = \beta_0 I_m$.
    \end{enumerate}
    Moreover, by \cref{lemma:Hartman_Grobman}, we know that each hyperbolic matrix $\widetilde{\Lambda}_{m, k}$ is Hurwitz as the system for $\tilde{v}_{m, t}$ is asymptotically stable.
    \begin{equation}
        \sup_m \max_{ i \in [m]}(\lambda_i(\widetilde{\Lambda}_{m, k})) < 0.
    \end{equation}
    This is the stability boundary for the state-space models under perturbations. 
    
    Therefore the original diagonal unperturbed recurrent weight matrix $\Lambda_m$ satisfies the following eigenvalue inequality \textbf{uniformly} in $m$. Since $\Lambda_m$ is diagonal:
    \begin{equation}
    \label{eq:bounded_away_from_0}
        \sup_m \max_{ i \in [m]}(\lambda_i(\Lambda_m)) \leq -\beta_0.
    \end{equation}

    Therefore the model memory decays exponentially uniformly
    \begin{align}
        \mathcal{M}(\widehat{\mathbf{H}}_m)(t) 
        & := \sup_{X_0} \frac{1}{X_0 + 1} \left |\frac{d}{dt} \hat{y}_{m, t} \right | \\
        & = \sup_{X_0} \frac{1}{X_0 + 1} |c_m^\top [\sigma'(h_{m, t}) \circ v_{m, t} ]| \\
        & \leq \sup_{X_0} \frac{1}{X_0 + 1} |c_m|_{2} |\sigma'(h_{m, t}) \circ v_{m, t}|_{2} \\
        & \leq \sup_{X_0} \frac{1}{X_0 + 1} |c_m|_{2} \cdot \sup_{z} |\sigma'(z)| \cdot |e^{-\beta_0 t} v_{m, 0} |_{2} \\
        & \leq \sup_{X_0} \frac{1}{X_0 + 1} \bigg (\sup_m |c_m|_{2} \cdot \sup_{z} |\sigma'(z)| \cdot V_0 \bigg ) e^{-\beta_0 t} \\
        & \leq \sup_{X_0} \frac{1}{X_0 + 1} \bigg (\sup_m |c_m|_{2} \cdot L_0 \cdot V_0 \bigg ) e^{-\beta_0 t} \\
        & \leq \sup_{X_0} \bigg (\sup_m |c_m|_{2} \cdot L_0 \\
        & \cdot \Big ( \frac{X_0}{X_0 + 1} d (\sup_m \|U_m\|_2) + \frac{1}{X_0 + 1} (\sup_m \|b_m\|_2) \Big ) \bigg ) e^{-\beta_0 t} \\
        & \leq \bigg (\sup_m |c_m|_{2} \cdot L_0 \Big (d \sup_m \|U_m\|_2 + \sup_m \|b_m\|_2 \Big ) \bigg ) e^{-\beta_0 t} \\
        & \leq (d+1) L_0 \theta_{\max}^2 e^{-\beta_0 t}
    \end{align}
    The inequalities are based on vector norm properties, Lipschitz continuity of $\sigma(z)$ and uniform boundedness of unperturbed initial conditions. 
    Therefore we know the model memories are uniformly decaying. 
    
    By \cref{lemma:model_decay_target_decay}, the target $\mathbf{H}$ has an exponentially decaying memory as it is approximated by a sequence of models $\{\widehat{\mathbf{H}}_m\}_{m=1}^\infty$ with uniformly exponentially decaying memory. 
\end{proof}

\begin{remark}
\label{remark:stability_boundary}
    When the approximation is unstable, we cannot have the real parts of the eigenvalues for recurrent weights bounded away from 0 in \cref{eq:bounded_away_from_0}. 
    As the stability of linear RNNs requires the real parts (of the eigenvalues) to be negative, then the maximum of the real parts will converge to 0. 
    This is the stability boundary of state-space models. 
    \begin{equation}
    \label{eq:unstable_consequence}
        \lim_{m \to \infty} \max_{ i \in [m]}(\lambda_i(\Lambda_m)) = 0^-.
    \end{equation}
\end{remark}

\begin{remark}
    The uniform weights bound is necessary in the sense that: 
        Since state-space models are universal approximators, they can approximate targets with long-term memories. 
    However, if the target has an non-exponential decaying (e.g. polynomial decaying) memory, the weights bound of the approximation sequence will be exponential in the sequence length $T$. 
    \begin{equation}
        \theta_{max}^2 \geq e^{\beta_0 T} \frac{\mathcal{M}(\mathbf{H})(T)}{(d+1) L_0}. 
    \end{equation}
    This result indicates that scaling up SSMs without reparameterization is inefficient in learning sequence relationships with a large $T$ and long-term memory. 
\end{remark}

\modification{
\begin{remark}[\textbf{On the generalization to multi-layer cases}]
\label{remark:generalization_to_multilayer}
    We will use the following two-layer state-space models to demonstrate the idea to generalize this result to multi-layer cases.
    \begin{align}
        \frac{dh_t}{dt} 
        & = \Lambda_1 h_t + U_{1} x_t \\
        y_t 
        & = \sigma(h_t) \\
        \frac{d s_t}{dt} 
        & = \Lambda_2 s_t + U_{2} y_t \\
        \hat{z}_t 
        & = c^\top \sigma(s_t)
    \end{align}
    We can have the following memory function bounds: For simplicity, we drop the term $m$ in $\Lambda_1, \Lambda_2, U_1, U_2$. 
    \begin{align}
        \mathcal{M}(\widehat{\mathbf{H}}_m)(t) 
        & := \sup_{X_0} \frac{1}{X_0 + 1} \left |\frac{d}{dt} \hat{z}_{m, t} \right |_2 \\
        & = \sup_{X_0} \frac{1}{X_0 + 1} \left | c^\top \left ( \sigma'(s_{m, t}) \circ \frac{d s_{m, t}}{dt} \right ) \right |_2 \\
        & = \sup_{X_0} \frac{1}{X_0 + 1} \left | c^\top \left ( \sigma'(s_{m, t}) \circ \int_{0}^t e^{\Lambda_2 t_1} U_2 \frac{dy_{m, t-t_1}}{dt} dt_1 \right ) \right |_2 \\
        & = \sup_{X_0} \frac{1}{X_0 + 1} \left | c^\top \left ( \sigma'(s_{m, t})\circ \int_{0}^t e^{\Lambda_2 t_1} U_2 (\sigma'(h_{m, t-t_1}) \circ v_{m,t-t_1}) dt_1 \right ) \right |_2 \\
        & = \sup_{X_0} \frac{1}{X_0 + 1} \left | c^\top \left ( \sigma'(s_{m, t}) 
        \circ \int_{0}^t e^{\Lambda_2 t_1} U_2 (\sigma'(h_{m, t-t_1}) \circ e^{\Lambda_1 (t-t_1)} v_{m, 0}) dt_1 \right ) \right |_2 \\
        & \leq \sup_{X_0} \frac{1}{X_0 + 1} |c |_2  | \sigma'(s_{m, t}) |_2 
        \int_{0}^t |e^{\Lambda_2 t_1}|_2 |U_2|_2 (|\sigma'(h_{m, t-t_1})|_2 |e^{\Lambda_1 (t-t_1)}|_2 V_0) dt_1 \\
        & \leq L_0^2 \theta_{max}^2 \sup_{X_0} \frac{1}{X_0 + 1} \int_{0}^t |e^{\Lambda_2 t_1}|_2 |e^{\Lambda_1 (t-t_1)}|_2 V_0 dt_1 \\
        & \leq L_0^2 \theta_{max}^3 \sup_{X_0} \frac{(d X_0 + 1)}{X_0 + 1} \int_{0}^t |e^{\Lambda_2 t_1}|_2 |e^{\Lambda_1 (t-t_1)}|_2 dt_1 \\
        & \leq L_0^2 \theta_{max}^3 \sup_{X_0} \frac{(d X_0 + 1)}{X_0 + 1} \int_{0}^t |e^{-\beta_0 t_1}|_2 |e^{-\beta_0 (t-t_1)}|_2 dt_1 \\
        & \leq (d+1) L_0^2 \theta_{max}^3 t e^{-\beta_0 t}.
    \end{align}
\end{remark}
The first inequality comes from the Cauchy inequality ($|a \circ b|_2 \leq |a|_2 \cdot |b|_2$).
The second inequality comes from the property of activation $\sigma(\cdot)$ and uniform bound on weights. 
The third inequality comes from the bound of $V_0$ in \cref{eq:bound_of_initial_conditions}.
The last inequality is the direct evaluation based on the eigenvalues of $\Lambda_1$ and $\Lambda_2$. 
As here is a fast decaying term $e^{-\beta_0 t}$, we simplify other polynomial scale components in $P$. 
}

A further generalization of the memory function for $\ell$-layer SSMs would be: For some polynomial $P(t)$ with degree at most $l-1$
\begin{equation}
    \mathcal{M}(\widehat{\mathbf{H}}_m)(t) \leq (d+1) L_0^{\ell} \theta_{max}^{\ell + 1} P(t) e^{-\beta_0 t}. 
\end{equation}

\subsection{Proof for \texorpdfstring{\cref{thm:stable_reparameterization_linear_functional}}{}}
\label{subsec:proof_for_stable_reparameterization_linear_functional}

\begin{proof}
Let the target linear functional be $H_t(\mathbf{x}) = \int_{-\infty}^t \rho(t-s) x_s ds$. 
Here $\rho$ is an $L_1$ integrable function. 
We consider a simplified model setting with only parameters $c$ and $w$. 
Let $c_i, w_i$ be the unperturbed weights and $\tilde{c}_i, \tilde{w}_i$ be the perturbed recurrent weights.
Similar to $\rho$ being $L_1$ integrable, we note that $\int_0^\infty |c_i e^{f(w_i)t}|dt = \frac{|c_i|}{|f(w_i)|}$.
To have a sequence of well-defined model, we require they are uniformly (in $m$) absolutely integrable:
\begin{equation}
\label{eq:abs_integrable}
    \sup_m \sum_{i=1}^m \frac{|c_i|}{|f(w_i)|} < \infty, \quad \sup_m \sum_{i=1}^m \frac{1}{|f(w_i)|} < \infty. 
\end{equation}

Based $|\tilde{w} - w|_2 \leq \beta$ and $|\tilde{c} - c|_2 \leq \beta$. 
We know the approximation error is 
\begin{align}
    E_m(\beta) 
    & = \sup_{|\tilde{w} - w|_2 \leq \beta, |\tilde{c} - c|_2 \leq \beta} \int_0^\infty \left | \sum_{i=1}^m \tilde{c}_i e^{f(\tilde{w}_i) t} - \rho(t) \right | dt \\
    & \leq \sup_{|\tilde{w} - w|_2 \leq \beta} \int_0^\infty \left | \sum_{i=1}^m c_i e^{f(w_i) t} - \rho(t) \right | dt \\
    & + \sup_{|\tilde{w} - w|_2 \leq \beta} \int_0^\infty \left | \sum_{i=1}^m c_i e^{f(\tilde{w}_i) t}  - \sum_{i=1}^m c_i e^{f(w_i) t} \right | dt \\
    & + \sup_{|\tilde{w} - w|_2 \leq \beta, |\tilde{c} - c|_2 \leq \beta} \int_0^\infty \left | \sum_{i=1}^m (\tilde{c_i} - c_i) e^{f(\tilde{w}_i) t}  \right | dt \\
    & \leq \sup_{|\tilde{w} - w|_2 \leq \beta} \int_0^\infty \left | \sum_{i=1}^m c_i e^{f(w_i) t} - \rho(t) \right | dt \\
    & + \sup_{|\tilde{w} - w|_2 \leq \beta} \int_0^\infty \left | \sum_{i=1}^m c_i e^{f(\tilde{w}_i) t}  - \sum_{i=1}^m c_i e^{f(w_i) t} \right | dt \\
    & + \sup_{|\tilde{w} - w|_2 \leq \beta, |\tilde{c} - c|_2 \leq \beta} \int_0^\infty \sum_{i=1}^m \beta |e^{f(\tilde{w}_i) t} - e^{f(w_i) t} + e^{f(w_i) t}| dt \\
    & \leq E_m(0) + \sup_{|\tilde{w} - w|_2 \leq \beta} \int_0^\infty \sum_{i=1}^m |c_i| \left | e^{f(\tilde{w}_i) t}  - e^{f(w_i) t} \right | dt \\
    & + \sup_{|\tilde{w} - w|_2 \leq \beta, |\tilde{c} - c|_2 \leq \beta} \int_0^\infty \beta \sum_{i=1}^m |e^{f(\tilde{w}_i) t} - e^{f(w_i) t} | dt + \int_0^\infty \beta \left | \sum_{i=1}^m e^{f(w_i) t}  \right | dt \\
    & = E_m(0) + \sup_{|\tilde{w} - w|_2 \leq \beta} \int_0^\infty \sum_{i=1}^m (|c_i| + \beta) \left | e^{f(\tilde{w}_i) t}  - e^{f(w_i) t} \right | dt \\
    & + \int_0^\infty \beta \left | \sum_{i=1}^m e^{f(w_i) t}  \right | dt \\
    & = E_m(0) + \sum_{i=1}^m (|c_i| + \beta) \sup_{|\tilde{w} - w|_2 \leq \beta} \int_0^\infty \left | e^{f(\tilde{w}_i) t}  - e^{f(w_i) t} \right | dt + \int_0^\infty \beta \left | \sum_{i=1}^m e^{f(w_i) t}  \right | dt \\
    & = E_m(0) + \sum_{i=1}^m (|c_i|+\beta) \sup_{|\tilde{w}_i - w_i| \leq \beta} \int_0^\infty \left | e^{f(\tilde{w}_i) t}  - e^{f(w_i) t} \right | dt + \beta \sum_{i=1}^m \frac{1}{|f(w_i)|}\\
    & \leq E_m(0) + \sum_{i=1}^m (|c_i|+\beta) \frac{g(\beta)}{|f(w_i)|} + \beta \sum_{i=1}^m \frac{1}{|f(w_i)|}\\
    & = E_m(0) + \sum_{i=1}^m  \frac{g(\beta) (|c_i| + \beta) + \beta}{|f(w_i)|}.
\end{align}
The first and third inequalities are triangular inequality. 
The second inequality comes from the fact that $|\tilde{w}_i - w_i| \leq |\tilde{w} - w|_2 \leq \beta$. 
The fourth inequality is achieved via the property of stable reparameterization: For some continuous function $g(\beta): [0, \infty) \to [0, \infty), g(0)=0$:
\begin{equation}
    \sup_w \left[ |f(w)| \sup_{|\tilde{w} - w| \leq \beta} \int_0^\infty \left | e^{f(\tilde{w}) t}  - e^{f(w) t} \right | dt \right ] \leq g(\beta).
\end{equation}

By definition of stable approximation, we know $\lim_{m \to \infty} E_m(0) = 0$. 
Also according to the requirement of the stable approximation in \cref{eq:abs_integrable}, we have
\begin{align}
    \lim_{\beta \to 0} E(\beta) 
    & = \lim_{\beta \to 0} \lim_{m \to \infty} E_m(\beta) \\
    & \leq \lim_{\beta \to 0} \lim_{m \to \infty} E_m(0) + \left (\sup_m \sum_{i=1}^m \frac{|c_i| + \beta}{|f(w_i)|} \right ) * \lim_{\beta \to 0} g(\beta) + \lim_{\beta \to 0} \beta * \left ( \sup_m \sum_{i=1}^m \frac{1}{|f(w_i)|} \right )\\
    & = 0 + 0 + 0 = 0 = E(0).
\end{align}

\end{proof}

\begin{remark}
\label{remark:verification_of_stable_reparameterization}
\textbf{Here we verify the reparameterization methods satisfy the definition of stable reparameterization.}

For exponential reparameterization $f(w) = - e^w, w \in \mathbb{R}$: 
\begin{equation}
    \sup_{|\tilde{w} - w| \leq \beta} \int_0^\infty \left | e^{f(\tilde{w}) t}  - e^{f(w) t} \right | dt = \frac{e^\beta -1}{|f(w)|}.
\end{equation}

For softplus reparameterization $f(w) = -\log(1+e^w), w \in \mathbb{R}$: 
Notice that $\exp(-\beta) \log(1+\exp(w)) \leq \sup_{|\tilde{w} - w| \leq \beta} \log(1+\exp(\tilde{w})) \leq \exp(\beta) \log(1+\exp(w)) $, 
\begin{equation}
    \sup_{|\tilde{w} - w| \leq \beta} \int_0^\infty \left | e^{f(\tilde{w}) t}  - e^{f(w) t} \right | dt \leq \frac{e^\beta-1}{|f(w)|}.
\end{equation}

For ``best'' reparameterization $f(w) = -\frac{1}{a w^2 + b}, w \in \mathbb{R}, a, b > 0$: Without loss of generality, let $w \geq 0$
\begin{align}
    \sup_{|\tilde{w} - w| \leq \beta} \int_0^\infty \left | e^{f(\tilde{w}) t}  - e^{f(w) t} \right | dt 
    & =  |a (w+\beta)^2 - a w^2| \\
    & \leq \frac{\frac{a(\beta^2 + 2\beta w)}{a w^2 + b}}{|f(w)|} \\
    & \leq \frac{\frac{a(\beta^2 + 2\beta w)}{b}}{|f(w)|}. 
\end{align}
Here $g(\beta) = \frac{a(\beta^2 + 2\beta w)}{b}$. 
The famous M\"untz--Sz\'asz theorem indicates that selecting any non-zero constant $a$ does not affect the universality of linear RNN. 

While for the case without reparameterization $f(w) = w, w<0$: For $0 \leq \beta < -w$,
\begin{equation}
    \sup_{|\tilde{w} - w| \leq \beta} \int_0^\infty \left | e^{f(\tilde{w}) t}  - e^{f(w) t} \right | dt = \frac{\beta}{(-w-\beta)(-w)} = \frac{\beta}{(-w-\beta)|f(w)|},
\end{equation}
Here $\lim_{w \to -\beta} \sup_{w} \frac{\beta}{-w-\beta} = \infty$, therefore the direct parameterization is not a stable reparameterization. 
\end{remark}

\begin{remark}[On the generalization of existence of stable approximation to nonlinear functionals]
\label{remark:existence_of_stable_approximation_for_nonlinear_functionals}
    The previous results are established for the stable approximation of linear functionals by linear RNNs with stable approximations. 

    Here we show that this can be further extended to nonlinear functionals. 
    According to the Volterra Series representation, the nonlinear functional has expansion by multi-layer composition or element-wise product~\citep{wang2023.StatespaceModelsLayerwisea}. 
    Therefore if the existence of stable approximation is preserved for functional composition and polynomial, then we can generalize the above argument to the nonlinear functionals by working with nonlinear functional representations.

    \begin{theorem}[\citet{boyd1984.AnalyticalFoundationsVolterra,wang2023.StatespaceModelsLayerwisea}]
        For any continuous time-invariant system with $x(t)$ as input and $y(t)$ as output can be expanded in the Volterra series as follow
        \begin{equation}
            y(t) = \rho_0 + \sum_{n=1}^N \int_0^t \cdots \int_0^t \rho_n(\tau_1, \dots, \tau_n) \prod_{j=1}^n x(t-\tau_j) d\tau_j.
        \end{equation}
        In particular, we call the expansion order $N$ to be the series' order. 
    \end{theorem}

    \begin{lemma}[Stable approximation induced by polynomials of stable approximation]
        Assume $\mathbf{H}_1$ and $\mathbf{H}_2$ can be stably approximated, let $f$ be some polynomial, then $f(\mathbf{H}_1, \mathbf{H}_2)$ can also be stably approximated. 
    \end{lemma}
    \begin{proof}
        Let $f(\mathbf{H}_1, \mathbf{H}_2) = \sum_{i, j} c_{i, j} \mathbf{H}_1^i \mathbf{H}_2^j$. 
        The definition of functional product is by the element-wise product: $(\mathbf{H}_1 \mathbf{H}_2) (\mathbf{x}) = \mathbf{H}_1(\mathbf{x}) \odot \mathbf{H}_2 (\mathbf{x})$. 
        \begin{align}
            E_m(\beta) 
            & = \sup_{|\tilde{\theta} - \theta| \leq \beta} \|f(\mathbf{H}_1, \mathbf{H}_2) - f(\mathbf{H}_1(\tilde{\theta}), \mathbf{H}_2(\tilde{\theta})) \|_{W^{1, \infty}}  \\
            & \leq E_m(0) + \sup_{|\tilde{\theta} - \theta| \leq \beta} \|f(\mathbf{H}_1(\theta), \mathbf{H}_2(\theta)) - f(\mathbf{H}_1(\tilde{\theta}), \mathbf{H}_2(\tilde{\theta})) \|_{W^{1, \infty}}  \\
            & \leq E_m(0) + \sup_{|\tilde{\theta} - \theta| \leq \beta} \|f(\mathbf{H}_1(\theta), \mathbf{H}_2(\theta)) - f(\mathbf{H}_1(\theta), \mathbf{H}_2(\tilde{\theta})) \|_{W^{1, \infty}}  \\
            & \qquad \qquad + \sup_{|\tilde{\theta} - \theta| \leq \beta} \|f(\mathbf{H}_1(\theta), \mathbf{H}_2(\tilde{\theta})) - f(\mathbf{H}_1(\tilde{\theta}), \mathbf{H}_2(\tilde{\theta})) \|_{W^{1, \infty}}  \\
            & \leq E_m(0) + \sum_{i \geq 0, j \geq 1} c_{i, j} j \|\widehat{\mathbf{H}}_1(\theta)\|_{W^{1, \infty}}^i (\|\widehat{\mathbf{H}}_2(\theta)\|_{W^{1, \infty}} + E_m^{\mathbf{H}_2}(\beta) )^{j-1} E_m^{\mathbf{H}_2}(\beta)\\
            & \qquad \qquad + \sum_{i \geq 1, j \geq 0} c_{i, j} i (\|\widehat{\mathbf{H}}_1(\theta)\|_{W^{1, \infty}} + E_m^{\mathbf{H}_1}(\beta))^{i-1} \|\widehat{\mathbf{H}}_2(\theta)\|_{W^{1, \infty}}^{j} E_m^{\mathbf{H}_1}(\beta).
        \end{align}
        Therefore $E(\beta) \leq \lim_{m \to \infty} E_m(\beta) < \infty$. 
        The third inequality comes from \cref{eq:property_of_nonlinear_functional_sequence_norm}.
    \end{proof}

\end{remark}

\subsection{Proof for \texorpdfstring{\cref{thm:stable_parameterization_opt}}{}}
\label{subsec:stable_parameterization_opt}

\begin{proof}
For any $1 \leq j \leq m$, assume the loss function we used is the $L_{\infty}$ norm: $\textrm{Loss} = \sup_t \|H_t - \widehat{H}_{m, t}\|_{\infty}$. 
Notice that by time-homogeneity, $\textrm{Loss} = \|H_t - \widehat{H}_{m, t}\|_{\infty}$ for any $t$. 
This loss function is larger than the common mean squared error, which is usually chosen in practice for the smoothness reason. 

\begin{align}
    \left |\frac{\partial \textrm{Loss}}{\partial w_j} \right | 
    & = \left | \frac{\partial \|H_t - \widehat{H}_{m, t}\|_{\infty}}{\partial w_j} \right |\\
    & = \left | \frac{\partial \sup_{\|\mathbf{x}\|_{\infty} \leq 1} |H_t(\mathbf{x}) - \widehat{H}_{m, t}(\mathbf{x})|}{\partial w_j} \right |\\
    & = \left | \frac{\partial \sup_{\|\mathbf{x}\|_{\infty} \leq 1} |\int_{-\infty}^t (\rho(t-s) - \sum_{i=1}^m c_i e^{-f(w_i) (t-s)} ) x_s ds| }{\partial w_j} \right |\\
    & = \left | \frac{\partial \int_{-\infty}^t |\rho(t-s) - \sum_{i=1}^m c_i e^{-f(w_i) (t-s)} | ds }{\partial w_j} \right |\\
    & = \left | \frac{\partial \int_{-\infty}^t |(\rho(t-s) - \sum_{i \neq j} c_i e^{-f(w_i) (t-s)}) - c_j e^{-f(w_j)(t-s)} | ds }{\partial w_j} \right |\\
    & = \left | \frac{\partial \int_0^{\infty} |(\rho(s) - \sum_{i \neq j} c_i e^{-f(w_i) s}) - c_j e^{-f(w_j) s} | ds }{\partial w_j} \right | \\
    & \leq \int_0^{\infty} \left | \frac{\partial | (\rho(s) - \sum_{i \neq j} c_i e^{-f(w_i) s}) - c_j e^{-f(w_j) s} |}{\partial w_j} \right | ds \\ 
    & \leq \int_0^{\infty} \left | \frac{\partial | c_j e^{-f(w_j) s} | }{\partial w_j} \right | ds
\end{align}

The first equality is the definition of the loss function. 
The second equality equality comes from the definition of the linear functional norm. 
The third equality expand the linear functional and linear RNNs into the convolution form. 
The fourth equality utilize the fact that we can manually select $x_t$'s sign to achieve the maximum value. 
The fifth equality is separating the term in dependent of variable $w_j$. 
The sixth equality is change of variable from $t-s$ to $s$. 
The inequality is triangular inequality. 
The last equality is dropping the term independent of variable $w_j$. 

\begin{align}
    \left |\frac{\partial \textrm{Loss}}{\partial w_j} \right | 
    & \leq \int_0^{\infty} \left | \frac{\partial | c_j e^{-f(w_j) s} | }{\partial w_j} \right | ds \\
    & = |c_j f'(w_j)| \int_0^{\infty} e^{-f(w_j) s} s \  ds \\
    & = \left | c_j \frac{f'(w_j)}{f(w_j)} \right | \int_0^{\infty} e^{-f(w_j) s} ds \\
    & = \left | c_j \frac{f'(w_j)}{f(w_j)^2} \right | (1 - \lim_{s \to \infty} e^{-f(w_j) s}) = \left | c_j \frac{f'(w_j)}{f(w_j)^2} \right | .
\end{align}

The first equality is evaluating the derivative. 
The second equality is extracting $|f'(w)|$ from integral. 
The third equality is doing the integration by parts. 

In particular, notice that $c_j$ is a constant independent of the recurrent weight parameterization $f$: 
\begin{equation}
    \widehat{H}_{m, t}(\mathbf{x}) = \int_{-\infty}^t \sum_{i=1}^m c_i e^{-f(w_i) (t-s)} x_s ds. 
\end{equation}
Therefore $c_j$ is a parameterization indepndent value, we will denote it by $C_{\mathbf{H}, \widehat{\mathbf{H}}_m}$.

Moreover, in the discrete setting, assume $h_{k+1}=f(w) \circ h_k + U x_k$, 
\begin{align}
    \left |\frac{\partial \textrm{Loss}}{\partial w_j} \right | 
    & \leq \sum_{k=0}^{\infty} \left | \frac{\partial | c_j f(w_j)^k | }{\partial w_j} \right | ds \\
    & = |c_j f'(w_j)| \sum_{k=1}^{\infty} k f(w_j)^{k-1} \\
    & = |c_j f'(w_j)| \left ( \sum_{k=1}^{\infty} f(w_j)^{k-1} \right )^2 \\
    & = \left | c_j \frac{f'(w_j)}{(1-f(w_j))^2} \right | .
\end{align}
So the gradient norm is bounded by
\begin{equation}
    \left | \frac{\partial \textrm{Loss}}{\partial w_j} \right | =  \frac{|c_j f'(w_j)|}{(1-f(w_j))^2}.
\end{equation}
\end{proof}

\paragraph{Nonlinear functionals}
Now we show the generalization into the nonlinear functional:
Consider the Volterra Series representation of the nonlinear functional. 
\begin{theorem}[\citep{boyd1984.AnalyticalFoundationsVolterra}]
    For any continuous time-invariant system with $x(t)$ as input and $y(t)$ as output can be expanded in the Volterra series as follow
    \begin{equation}
        y(t) = h_0 + \sum_{n=1}^N \int_0^t \cdots \int_0^t h_n(\tau_1, \dots, \tau_n) \prod_{j=1}^n x(t-\tau_j) d\tau_j.
    \end{equation}
    Here $N$ is the series' order. Linear functional is an order-1 Volterra series. 
\end{theorem}

For simplicity, we will only discuss the case for $N=2$. 
When we take the Hyena approach \citep{poli2023.HyenaHierarchyLarger} and approximate the order-2 kernel $h_2(\tau_1, \tau_2)$ with its rank-1 approximation: 
\begin{equation}
    h_2(\tau_1, \tau_2) = h_{2,1}(\tau_1) h_{2,2}(\tau_2). 
\end{equation}
Here $h_{2,1}$ and $h_{2,2}$ are again order-1 kernel which can be approximated with linear RNN's kernel. 
In other words, the same gradient bound also holds for general nonlinear functional with the following form: 
\begin{equation}
    G_{f}(w) := \left | \frac{\partial E}{\partial w} \right | = C_{\mathbf{H}, \widehat{\mathbf{H}}_m} \frac{|f'(w)|}{f(w)^2}.
\end{equation}
And the discrete version is
\begin{equation}
    G^D_{f}(w) := \left | \frac{\partial E}{\partial w} \right | = C_{\mathbf{H}, \widehat{\mathbf{H}}_m} \frac{|f'(w)|}{(1-f(w))^2}.
\end{equation}

\subsection{Lemmas}

\begin{lemma}
\label{lemma:strictly_increasing}
    If the activation $\sigma(\cdot)$ is bounded, strictly increasing, continuously differentiable function over $\mathbb{R}$. 
    Then for all $C>0$, there exists $\epsilon_C$ such that $\forall |z| \leq C_\epsilon$, $|\sigma'(z)| \geq \epsilon_C$. 
\end{lemma}

\begin{proof}
    Since $\sigma(\cdot)$ is monotonically increasing, therefore $\sigma'(\cdot) > 0, \forall z \geq 0$. 
    Notice that $\sigma'(\cdot)$ is continuous, for any $C > 0$, we know $\frac{1}{2} \min_{|z| \leq C} \sigma'(z) > 0$.
    Define $\epsilon_C := \frac{1}{2} \min_{|z| \leq C} \sigma'(z) > 0$, it can be seen the target statement is satisfied. 
\end{proof}

\begin{lemma}
\label{lemma:decaying_v}
    Assume the target functional sequence has a $\beta_0$-stable approximation and the perturbed model has a decaying memory,
    we show that $\tilde{v}_{m, t} \to 0$ for all $m$.
\end{lemma}
\begin{proof}
    For any $m$, fix $\widetilde{\Lambda}_m$ and $\widetilde{U}_m$. 
    Since the perturbed model has a decaying memory, 
    \begin{equation}
        \lim_{t \to \infty} \left | \frac{d}{dt}\widetilde{H}_m (\mathbf{u}^x) \right | = \lim_{t \to \infty} \left | c^\top (\sigma'(\tilde{h}_{m, t}) \circ \frac{d\tilde{h}_{m, t}}{dt}) \right |  = \lim_{t \to \infty} \left | c^\top (\sigma'(\tilde{h}_{m, t}) \circ \tilde{v}_{m, t}) \right | = 0.
    \end{equation}

    By linear algebra, there exist $m$ vectors $\{\Delta c_i\}_{i=1}^m$, $|\Delta c_i|_{\infty} < \beta$ such that $c_m+ \Delta c_1$, \dots, $c_m + \Delta c_m$ form a basis of $\mathbb{R}^m$.
    We can then decompose any vector $u$ into
    \begin{equation}
        u = k_{1} (c_m+ \Delta c_1) + \cdots + k_{m} (c_m+ \Delta c_m).
    \end{equation}
    Take the inner product of $u$ and $\tilde{v}_{m, t}$, we have
    \begin{equation}
        \lim_{t \to \infty} u^\top (\sigma'(\tilde{h}_{m, t}) \circ \tilde{v}_{m, t}) = \sum_{i=1}^m k_i \lim_{t \to \infty} (c_m+ \Delta c_i)^\top (\sigma'(\tilde{h}_{m, t}) \circ \tilde{v}_{m, t})  = 0
    \end{equation}
    As the above result holds for any vector $u$, we get
    \begin{equation}
        \lim_{t \to \infty} \left | \sigma'(\tilde{h}_{m, t}) \circ \tilde{v}_{m, t} \right |_{\infty} = 0.
    \end{equation}

    As required in \cref{eq:uniformly_bounded_hiddens}, the hidden states are uniformly (in $m$) bounded over bounded input sequence. 
    There exists constant $C_0>0$ such that 
    \begin{equation}
        \sup_{m, t} |h_{m, t}|_{\infty} < C_0.
    \end{equation}
    Since $\sigma$ is continuously differentiable and strictly increasing, by \cref{lemma:strictly_increasing}, there exists $\epsilon_{C_0} > 0$ such that 
    \begin{equation}
        |\sigma'(z)| > \epsilon_{C_0}, \quad \forall |z| \leq C_0. 
    \end{equation}
    Therefore 
    \begin{equation}
        \sup_t \left |\sigma'(\tilde{h}_{m, t}) \right |_{\infty} > \epsilon_{C_0}.
    \end{equation}

    We get
    \begin{equation}
        \lim_{t \to \infty} |\tilde{v}_{m, t}|_{\infty} = 0.
    \end{equation}
\end{proof}

\begin{lemma}
    \label{lemma:Hartman_Grobman}
    Consider a dynamical system with the following dynamics: $h_0=0$
    \begin{equation}\label{eq:hartman_grobman}
        \begin{aligned}
            \frac{dv_{t}}{dt} & = \Lambda v_{t},  \\
            v_{0}             & = \Lambda h_0 + \widetilde{U} x_0 + \tilde{b} = \widetilde{U} x_0 + \tilde{b}.
        \end{aligned}
    \end{equation}
    If $\Lambda \in \mathbb{R}^{m \times m}$ is diagonal, hyperbolic and the system in \cref{eq:hartman_grobman} is satisfies $\lim_{t \to \infty} v_t = 0$ over any bounded Heaviside input $\mathbf{u}^{x_0}, |x_0|_\infty < \infty$,
    then the matrix $\Lambda$ is Hurwitz.
\end{lemma}

\begin{proof}
    By integration we have the following explicit form:
    \begin{equation}
        v_t = e^{\Lambda t} v_0 = e^{\Lambda t} (\widetilde{U} x_0 + \tilde{b}).
    \end{equation}
    The stability requires $\displaystyle \lim_{t \to \infty} |v_t| = 0$ for all inputs $v_0 = \widetilde{U}x_0 + \tilde{b}$. 
    Notice that with perturbation from $\tilde{U}$ and $\tilde{b}$, the set of initial points $\{v_0\}$ is m-dimensional.
    Therefore the matrix $\Lambda$ is Hurwitz in the sense that all eigenvalues' real parts are negative. 
\end{proof}

\begin{lemma}
    \label{lemma:model_decay_target_decay}
    Consider a continuous function $f: [0, \infty) \to \mathbb{R}$, assume it can be approximated by a sequence of continuous functions $\{f_m\}_{m=1}^\infty$ universally:
    \begin{equation}
        \label{eq:universal_approximation}
        \lim_{m \to \infty} \sup_{t \geq 0} |f(t) - f_m(t)| = 0.
    \end{equation}
    Assume the approximators $f_m$ are uniformly exponentially decaying with the same $\beta_0 > 0$:
    \begin{equation}
        \lim_{t \to \infty} \sup_{m \in \mathbb{N}_+} e^{\beta_0 t} |f_m(t)| \to 0.
    \end{equation}
    Then the function $f$ is also decaying exponentially:
    \begin{equation}
        \lim_{t \to \infty} e^{\beta t} |f(t)| \to 0, \quad \forall 0 < \beta < \beta_0.
    \end{equation}
\end{lemma}

The proof is the same as Lemma A.11 from \citep{wang2023.InverseApproximationTheory}. 
For completeness purpose, we attach the proof here: 

\begin{proof}
    Given a function $f \in C([0, \infty))$, we consider the transformation $\mathcal{T}f : [0, 1] \to \mathbb{R}$ defined as:
    \begin{equation}
        (\mathcal{T}f)(s) = \left\{\begin{array}{lcl} 0, & & {s = 0}\\ \frac{f(-\frac{\log s}{\beta_0})}{s},& & {s \in (0, 1].} \end{array} \right.
    \end{equation}
    Under the change of variables $s = e^{-\beta_0 t}$, we have:
    \begin{equation}
        f(t) = e^{-\beta_0 t} (\mathcal{T}f) (e^{-\beta_0 t}), \quad t \geq 0.
    \end{equation}
    According to uniformly exponentially decaying assumptions on $f_m$:
    \begin{equation}
        \lim_{s \to 0^+} (\mathcal{T} f_m)(s) 
        = \lim_{t \to \infty} \frac{f_m(t)}{e^{-\beta_0 t}} 
        = \lim_{t \to \infty} e^{\beta_0 t} f_m(t) 
        = 0,
    \end{equation}
    which implies $\mathcal{T} f_m \in C([0, 1])$. 
    
    For any $\beta < \beta_0$, let $\delta = \beta_0 - \beta > 0$. Next we have the following estimate
    \begin{align}
        & \sup_{s \in [0, 1]} \left | (\mathcal{T} f_{m_1})(s) - (\mathcal{T} f_{m_2})(s) \right | \\
        = & \sup_{t \geq 0} \left | \frac{f_{m_1}(t)}{e^{-\beta t}} - \frac{f_{m_2}(t)}{e^{-\beta t}} \right | \\
        \leq & \max \left \{ \sup_{0 \leq t \leq T_0} \left | \frac{f_{m_1}(t)}{e^{-\beta t}} - \frac{f_{m_2}(t)}{e^{-\beta t}} \right | , C_0 e^{-\delta T_0} \right \} \\
        \leq & \max \left \{ e^{\beta T_0} \sup_{0 \leq t \leq T_0} \left |f_{m_1}(t) - f_{m_2}(t) \right | , C_0 e^{-\delta T_0} \right \}
    \end{align}
    where $C_0$ is a constant uniform in $m$.
    
    For any $\epsilon>0$, take $T_0 = - \frac{\ln(\frac{\epsilon}{C_0})}{\delta},$ we have $C_0 e^{-\delta T_0} \leq \epsilon$. 
    For sufficiently large $M$ which depends on $\epsilon$ and $T_0$, by universal approximation (\cref{eq:universal_approximation}), we have $\forall m_1, m_2 \geq M$,
    \begin{align}
        \sup_{0 \leq t \leq T_0} \left |f_{m_1}(t) - f_{m_2}(t) \right | & \leq e^{-\beta T_0} \epsilon, \\
        e^{\beta T_0} \sup_{0 \leq t \leq T_0} \left |f_{m_1}(t) - f_{m_2}(t) \right | & \leq \epsilon.
    \end{align}
    Therefore, $\{f_m\}$ is a Cauchy sequence in $C([0, \infty))$.

    Since $\{f_m\}$ is a Cauchy sequence in $C([0, \infty))$ equipped with the sup-norm, using the above estimate we can have$\{\mathcal{T} f_m\}$ is a Cauchy sequence in $C([0, 1])$ equipped with the sup-norm.
    By the completeness of $C([0, 1])$, there exists $f^* \in C([0, 1])$ with $f^*(0) = 0$ such that
    \begin{equation}
        \lim_{m \to \infty} \sup_{s \in [0, 1]} |(\mathcal{T} f_m)(s) - f^*(s)| = 0.
    \end{equation}
    Given any $s > 0$, we have
    \begin{equation}
        f^*(s) = \lim_{m \to \infty} (\mathcal{T} f_m)(s) = (\mathcal{T} f)(s),
    \end{equation}
    hence
    \begin{equation}
        \lim_{t \to \infty} e^{\beta t} f(t) = \lim_{s \to 0^+} (\mathcal{T} f)(s) = f^*(0) = 0.
    \end{equation}
\end{proof}

\section{Motivation for the gradient-over-weight Lipschitz criterion}
\label{sec:motivation_for_gow}
Here we discuss the motivation for adopting the gradient-over-weight boundedness as the criterion for ``best-in-stability'' reparameterization.
First of all, the ``best'' reparameterization is proposed to further improve the optimization stability across memory patterns with different decays. 
The criterion ``gradient is Lipschitz to the weight'' is a necessary condition for the stability in the following sense:
\begin{enumerate}
    \item Consider functions $f(x) = x^4$, the gradient function $\frac{d f}{d x}(x) = 4x^3$ does not have a global Lipschitz coefficient for all input values $x$. Therefore for any fixed positive learning rate $\eta$, there exists an initial point $x_0$ (for example $x_0 = \sqrt{\frac{1}{2\eta}} + 1$) such that the convergence from initial point $x_0$ cannot be achieved via the gradient descent step 
    \begin{equation}
    \label{eq:gradient_descent_step}
        x_{k+1} = x_k - \eta g(x_k).
    \end{equation}
    It can be verified the convergence does not hold as $|x_{k+1}| > |x_k|$ for all $k$ when $x_0 = \sqrt{\frac{1}{2\eta}} + 1$. This comes from the fact that $|x_k| \geq \sqrt{\frac{1}{2\eta}}, \eta g(x) \geq 2 x_k$ hold for all $k$.
    \item Consider functions $f(x) = x^2$, the gradient function $g(x) = 2x$ is associated with a Lipschitz constant $L=2$. Then the same gradient descent step converges for any $\eta \leq \frac{1}{2}$ in \cref{eq:gradient_descent_step}. 
    \item As can be seen in the above two examples, \textbf{the criterion ``gradient is Lipschitz to the weight'' is associated with the convergence under large learning rate.} 
    As the use of larger learning rate is usually associated with faster convergence \citep{smith2019.SuperconvergenceVeryFast}, smaller generalization errors \citep{li2019.ExplainingRegularizationEffect}, we believe the Lipschitz criterion is a suitable stability criterion for the measure of optimization stability. 
    \item The gradient-over-weight ratio evaluated in \cref{fig:grad_norm_stable_parameterization_language_model}(a) is a numerical verification of our Theorem 3.4. The gradients of stable reparameterizations are less susceptible to the well-known issue of exploding or vanishing gradients \citep{bengio1994.LearningLongtermDependencies,hochreiter1998.VanishingGradientProblem}.
\end{enumerate}

\begin{table}[htb!]
    \caption{Summary of reparameterizations and corresponding gradient norm functions in continuous and discrete time. Notice that the $G_f$ and $G_f^D$ are rescaled up to a constant $C_{\mathbf{H}, \widehat{\mathbf{H}}}$.}
    \label{table:continuous_discrete_different_stable_parameterization_gradient_norm}
    \centering
    \begin{tabular}{c|c|c|c}
    \toprule
                & Reparameteriations& $f$               & $G_f$ or $G_f^D$ \\
    \midrule
    Continuous  & ReLU              & $-\textrm{ReLU}(w)$   & $\frac{1}{w^2} \1_{\{w > 0\}}$\\
                & Exp               & $-\exp(w)$            & $e^{-w}$ \\
                & Softplus          & $-\log(1+\exp(w))$    & $\frac{\exp(w)}{(1+\exp(w)) \log(1+\exp(w))^2}$ \\
                & ``Best''(Ours)    & $-\frac{1}{a w^2 + b}, a > 0, b > 0$      & $2a |w|$ \\
    \midrule
    Discrete    & ReLU              & $\exp(-\textrm{ReLU}(w))$                 &  $\frac{\exp(-w)}{(1-\exp(-w))^2} \1_{\{w > 0\}}$\\
                & Exp               & $\exp(-\exp(w))$                          &  $\frac{\exp(w-\exp(w))}{(1-\exp(-\exp(w)))^2}$\\
                & Softplus          & $\frac{1}{1+\exp(w)}$                     &  $e^{-w}$\\
                & Tanh              & $\tanh(w) = \frac{e^{2w} - 1}{e^{2w}+1}$  &  $e^{2w}$\\
                & ``Best''(Ours)    & $1-\frac{1}{w^2 + 0.5} \in (-1, 1)$       &  $2|w|$ \\
    \bottomrule
    \end{tabular}
\end{table}

\begin{figure}[t!]{
    \centering
    \subfigure[][Continuous time $\frac{G_f(w)}{|w|}$]{
        \includegraphics[width=0.42\textwidth]{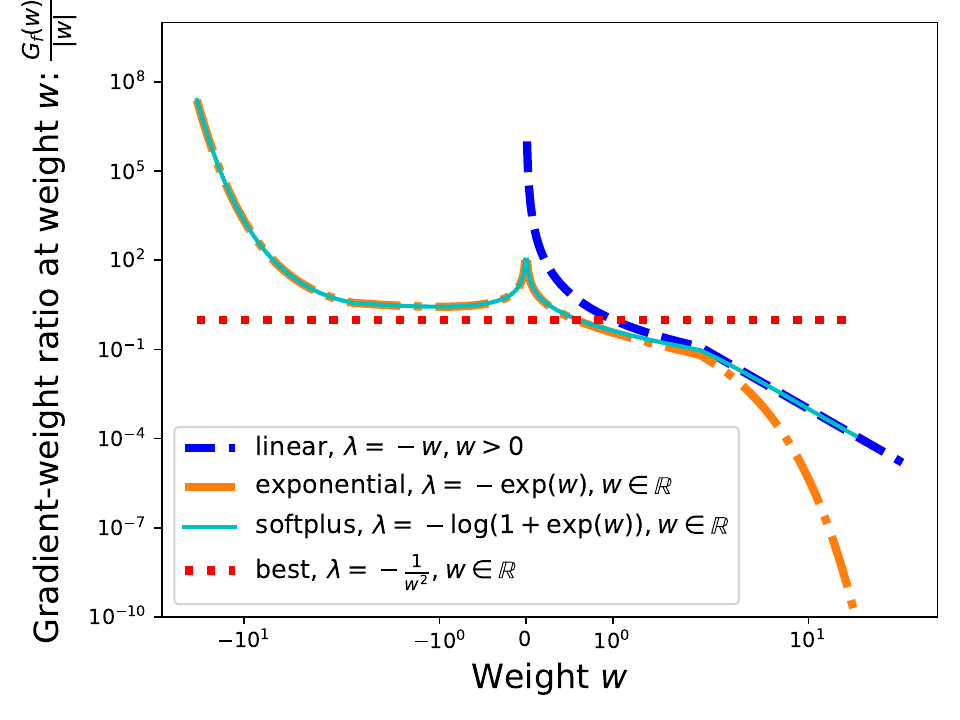}
    }
    \subfigure[][Discrete time $\frac{G_f^D(w)}{|w|}$]{
        \includegraphics[width=0.42\textwidth]{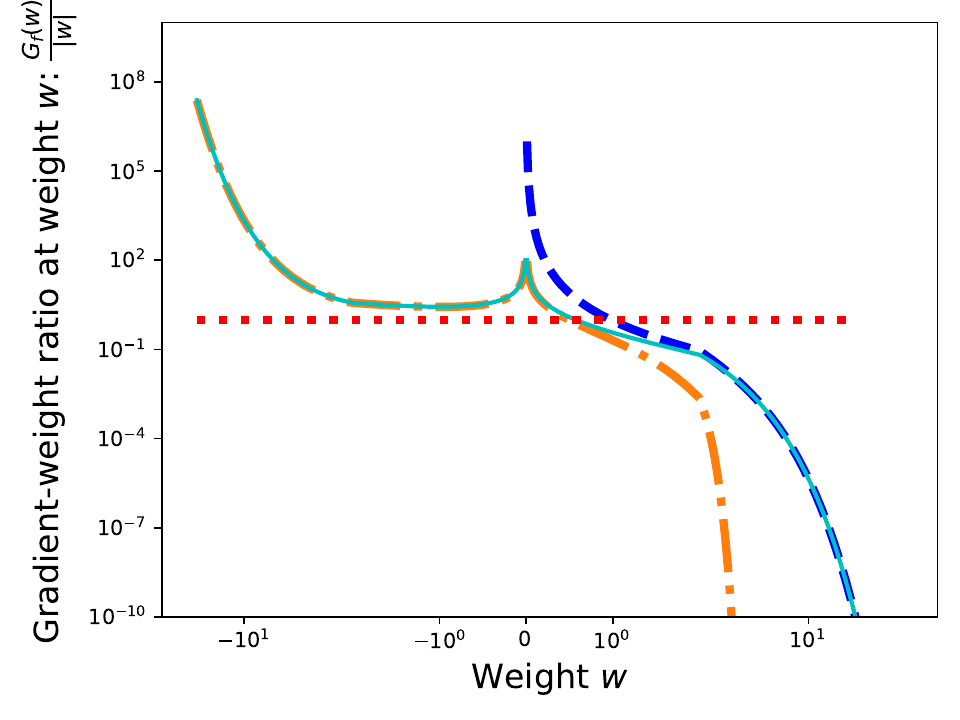}
    }
    \caption{
        Gradient norm function $G_f$ and $G_f^D$ of different parameterization methods.
        The ``best'' parameterization methods maintain a balanced gradient-over-weight ratio. 
        }
    \label{fig:grad_norm_stable_parameterization}
}
\end{figure}

\section{Comparison of different recurrent weights parameterization schemes}
\label{subsec:Comparison_of_different_recurrent_weights}

Here we evaluate the gradient norm bound function $G_f$ and $G_f^D$ for different parameterization schemes in \cref{table:continuous_discrete_different_stable_parameterization_gradient_norm} and \cref{fig:grad_norm_stable_parameterization}. 

\paragraph{On the Scenarios Where ``Best'' Parameterization is Preferable}
There is no guarantee that the ``best'' parameterization will outperform the Exp/Softplus parameterizations when all models exhibit good training stability. 
When the learning rate has been finetuned (at 5e-4) for CIFAR10, the optimal performance from ``best'' parameterization is worse than exp parameterization. 
This outcome is expected since this paper focuses on training stability rather than generalization. 
The key insight from Tables 1 and 2 is that the ``best'' parameterization offers a theoretically grounded alternative to the exp/softplus parameterizations.

\section{Numerical details}
\label{appendix:numerical_details}

In this section, the details of numerical experiments are provided for the completeness and reproducibility. 

\subsection{Synthetic task}

We conduct the approximation of linear functional with linear RNNs in the one-dimensional input and one-dimensional output case.
The synthetic linear functional is constructed with the polynomial decaying memory function is $\rho(t) = \frac{1}{(t+1)^{1.1}}$. 
Sequence length is 100. 
Total number of synthetic samples is 153600. 
The learning rate used is 0.01 and the batch size is 512. 

The perturbation list $\beta \in [0, 10^{-3}, 10^{-3} * 2^{1/2}, 10^{-3} * 2^{2/2}, \dots, 10^{-3} * 2^{20/2}]$.
Each evaluation of the perturbed error is sampled with 30 different weight perturbations to reduce the variance. 

\subsection{Language models}

The language modeling is done over WikiText-103 dataset \citep{merity2016.PointerSentinelMixtureb}. 
The model we used is based on the Hyena architecture with simple real-weights state-space models as the mixer \citep{poli2023.HyenaHierarchyLarger,smith2023.SimplifiedStateSpace}. 
The batch size is 16, total steps 115200 (around 16 epochs), warmup steps 1000.
The optimizer used is AdamW and the weight decay coefficient is 0.25. 
The learning rate for the recurrent layer is 0.004 while the learning rate for other layers are 0.005. 

\begin{figure}[t!]{
\centering
\subfigure[][lr=0.002, ``best'' reparameterization is also not optimal, but the final loss is comparable against Exp and Softplus]{
\includegraphics[width=0.5\textwidth]{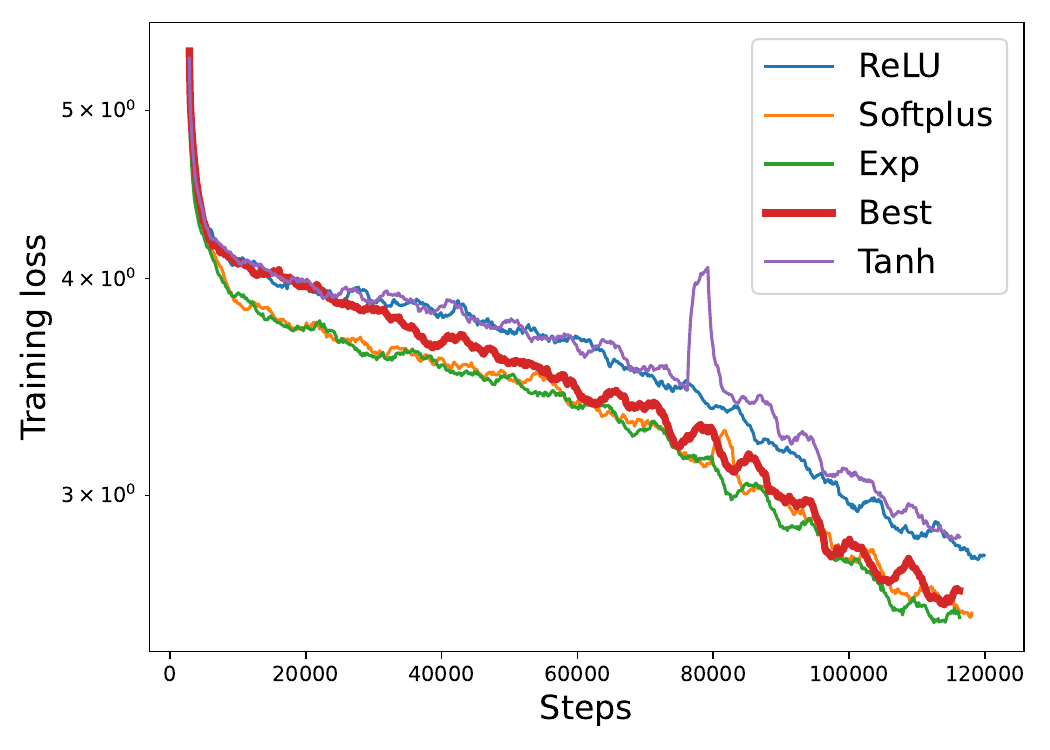}
}
\subfigure[][lr=0.01,  ``best'' reparameterization achieve the smallest loss]{
\includegraphics[width=0.5\textwidth]{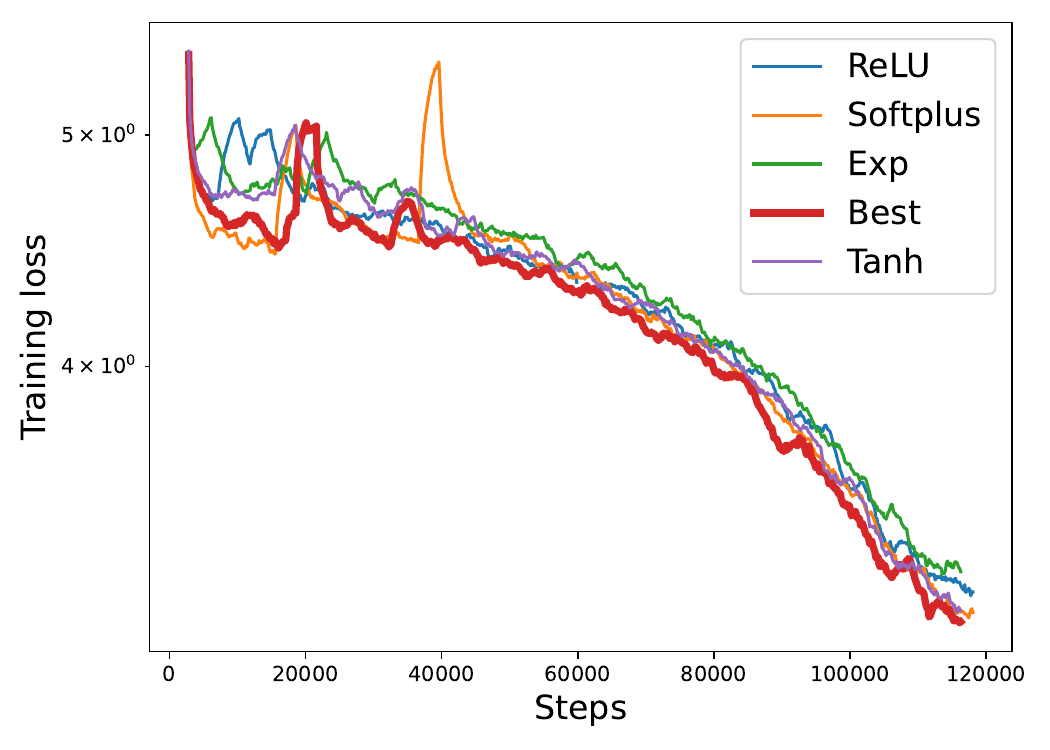}
}
\caption{
    The stability advantage of ``best'' reparameterization (red line) is usually better when the learning rate is larger.
}
\label{fig:1LR2LR5LR}
}
\end{figure}

In the main paper, we provide the training loss curve for learning rate = 0.005 as the stability of ``best'' discrete-time parameterization $f(w) = 1-\frac{1}{w^2+0.5}$ is mostly significant as the learning rate is large.
In \cref{fig:1LR2LR5LR}, we further provide the results for other learning rates (lr = 0.002, 0.010).
Despite the final loss not being optimal for the ``best'' reparameterization, it is observed that the training process exhibits enhanced stability compared to other parameterization methods.

\modification{
\subsection{On the stability of ``best'' reparameterization for large models}
\label{sec:stability_best_reprameterization}
The previous experiment on WikiText-103 language modelling shows the performance of stable reparameterization over the unstable cases. 
We further verify the optimization stability of ``best'' reparameterization in the following extreme setting. 
We construct a large scale language model with 3B parameters and train with larger learning rate (lr=0.01).
As can be seen in the following table, the only convergent model is the model with ``best'' reparameterization. 
We emphasize that the only difference between these models are the parameterization schemes for recurrent weights. 
Therefore the best reparameterization is the most \textbf{stable} parameterization. 
(We repeats the experiments with different seeds for three times.)
\begin{table}[ht!]
    \centering
    \begin{tabular}{c|cccc}
    & ``Best'' & Exp & Softplus & Direct \\
    \hline
    Convergent / total experiments & 3/3 & 0/3 & 0/3 & 0/3
    \end{tabular}
    \caption{Experiment to the stability of ``best'' reparameterization over lr = 0.01. All other reparameterizations diverged within 100 steps while the ``best'' reparameterizations can be used to train the model.}
    \label{tab:3B}
\end{table}
}

\subsection{Additional numerical results for associative recalls}
\label{subsec:associative_recall}
In this section, we study the performance of  of different stable reparameterizations over the extremely long sequences (up to 131k). 
It can be seen in \cref{tab:associative_recall} that stable parameterizations are better than the case without reparameterization and simple clipping. 
The advantage is more significant when the sequence length is longer. 
The models are trained under the exactly same hyperparameters.
\begin{table}[ht!]
\centering
\begin{tabular}{c|cc|cc}
Reparameterizations& Train acc, T=20 &Test acc, T=20 &Train acc, T=131k &Test acc, T=131k  \\
\hline
\textbf{``Best''}    & \textbf{57.95}            & \textbf{99.8}        & \textbf{53.57}          & \textbf{100} \\
Exp(S5)     & 54.55            & \textbf{99.8}        & \textbf{53.57}          & \textbf{100}       \\
Clip        & 50.0             & 76.6            & 13.91          & 9.4           \\
Direct      & 43.18            & 67.0            & 16.59          & 5.6         
    \end{tabular}
    \caption{Comparison of parameterizations on associative recalls. The first two columns are the train and test accuracy over \textbf{sequence length 20}, vocabulary size 10, while the second two columns are the train and test accuracy over \textbf{sequence length 131k} and vocabulary size 30.}
    \label{tab:associative_recall}
\end{table}

\end{document}